\documentclass{sig-alternate-05-2015}

\CopyrightYear{2017} \setcopyright{acmcopyright}
\conferenceinfo{SIGMOD '17,}{May 14--19, 2017, Chicago, IL, USA.}
\isbn{978-1-4503-4197-4/17/05}\acmPrice{\$15.00}
\doi{http://dx.doi.org/XXXX.XXXX} 

\usepackage{graphicx} %
\usepackage{subcaption} 
\usepackage{hyperref}
\usepackage{graphicx}
\usepackage{algorithm}
\usepackage{algorithmic}
\usepackage{epstopdf}

\usepackage{multirow}
\usepackage{hhline}

\newcommand{\Lap}{\text{Lap}}
\newcommand{\card}[1]{\texttt{card}(#1)}

\usepackage{xcolor}
\newcommand{\kcx}[1]{}
\newcommand{\kc}[1]{}

\newcommand{\shsit}[1]{{\itshape}}

\def\calS{\mathcal{S}}
\def\calQ{\mathcal{Q}}
\def\calX{\mathcal{X}}

\def\ttheta{\tilde{\theta}}

\def\g{g}
\def\dataolder{\text{older woman}}
\def\dataoverw{\text{overweight woman}}
\def\dataolders{\text{older women}}
\def\dataoverws{\text{overweight women}}

\newtheorem{theorem}{Theorem}[section]
\newtheorem{definition}[theorem]{Definition}
\newtheorem{lemma}[theorem]{Lemma}

\def\RQ{{R\cup Q}}

\def\mqm{\text{Markov Quilt Mechanism}}

\def\Pr{P}
\def\effect{max-influence}

\def\gT{\g_{\Theta}}
\def\piT{\pi^{\min}_{\Theta}}

\def\eT{e_{\Theta}}

\newcommand{\mypara}[1]{\medskip\noindent{\textbf{#1}}}
\def\calG{\mathcal{G}}
\def\bbR{\mathbb{R}}

\def\mgroupdp{\text{GroupDP}}
\def\minferen{\text{GK16}}
\def\mmqmapprox{\text{MQMApprox}}
\def\mmqmexact{\text{MQMExact}}
\def\mdp{\text{DP}}
\def\calP{\mathcal{P}}

\newcommand{\kset}[1]{\{1,\dots,#1\}}

\usepackage{sidecap}

\begin{document}

\title{Pufferfish Privacy Mechanisms for Correlated Data}

\numberofauthors{3} %
\author{
\alignauthor
Shuang Song\\
       \affaddr{UC San Diego}\\
       \email{shs037@eng.ucsd.edu}
\alignauthor
Yizhen Wang\titlenote{The first two authors contributed equally.}\\
       \affaddr{UC San Diego}\\
       \email{yiw248@eng.ucsd.edu}
\alignauthor 
Kamalika Chaudhuri\\
       \affaddr{UC San Diego}\\
       \email{kamalika@eng.ucsd.edu}
}

\maketitle
\begin{abstract}
Many modern databases include personal and sensitive correlated data, such as private information on users connected together in a social network, and measurements of physical activity of single subjects across time. However, differential privacy, the current gold standard in data privacy, does not adequately address privacy issues in this kind of data.

This work looks at a recent generalization of differential privacy, called Pufferfish, that can be used to address privacy in correlated data. The main challenge in applying Pufferfish is a lack of suitable mechanisms. We provide the first mechanism -- the Wasserstein Mechanism -- which applies to any general Pufferfish framework. Since this mechanism may be computationally inefficient, we provide an additional mechanism that applies to some practical cases such as physical activity measurements across time, and is computationally efficient. Our experimental evaluations indicate that this mechanism provides privacy and utility for synthetic as well as real data in two separate domains. 
\end{abstract}

\vspace{-10pt}
\ccsdesc[500]{Security and privacy~Privacy protections}

\def\parent{\texttt{parent}}
\printccsdesc

\vspace{-10pt}
\keywords{privacy, differential privacy, Pufferfish privacy}

\vspace{-10pt}
\section{Introduction}

Modern database applications increasingly involve personal data, such as healthcare, financial and user behavioral information, and consequently it is important to design algorithms that can analyze sensitive data while still preserving privacy. For the past several years, differential privacy~\cite{DMNS06} has emerged as the gold standard in data privacy, and there is a large body of work on differentially private algorithm design that apply to a growing number of queries~\cite{hay2010boosting,li2010optimizing, CMS11, SCS13, li2012adaptive,stoddard2014differentially,he2015dpt,chen2011publishing,xiao2010differentially}. The typical setting in these works is that each (independently drawn) record is a single individual's private value, and the goal is to answer queries on a collection of records while adding enough noise to hide the evidence of the participation of a single individual in the data.

Many modern database applications, such as those involving healthcare, power usage and building management, also increasingly involve a different setting -- correlated data -- privacy issues in which are not as well-understood. Consider, for example, a simple time-series application -- measurement of physical activity of a {\em{single subject}} across time. The goal here is to release aggregate statistics on the subject's activities over a long period (say a week) while hiding the evidence of an activity at any specific instant (say, 10:30am on Jan 4). If the measurements are made at small intervals, then the records form a {\em{highly correlated time-series}} as human activities change slowly over time. 

What is a good notion of privacy for this example? Since the data belongs to a {\em{single}} subject, differential privacy is not directly applicable; however, a modified version called {\em{entry-privacy}}~\cite{HR13} applies. Entry-privacy ensures that the inclusion of a single time-series entry does not affect the probability of any outcome by much. It will add noise with standard deviation $\approx 1$ to each bin of the histogram of activities. While this is enough for independent activities, it is insufficient to hide the evidence of correlated activities that may continue for several time points. A related notion is group differential privacy~\cite{DR14}, which extends the definition to participation of {\em{entire groups}} of correlated individuals or entries. Here, all entries are correlated, and hence group differential privacy will add $\approx O(T)$ noise to a histogram over $T$ measurements, thus destroying all utility. Thus to address privacy challenges in this kind of data, we need a different privacy notion.

A generalized version of differential privacy called Pufferfish was proposed by~\cite{KM12}. In Pufferfish, privacy requirements are specified through three components -- $\calS$, a set of secrets that represents what may need to be hidden, $\calQ$, a set of secret pairs that represents pairs of secrets that need to be indistinguishable to the adversary and $\Theta$, a class of distributions that can plausibly generate the data. Privacy is provided by ensuring that the secret pairs in $\calQ$ are indistinguishable when data is generated from any $\theta \in \Theta$. In the time-series example, $\calS$ is the set of activities at each time $t$, and secret pairs are all pairs of the form (Activity $a$ at time $t$, Activity $b$ at time $t$). Assuming that activities transition in a Markovian fashion, $\Theta$ is a set of Markov Chains over activities. Pufferfish captures correlation in such applications effectively in two ways -- first, unlike differential privacy, it can hide private values against correlation across multiple entries/individuals; second, unlike group privacy, it also allows utility in cases where a large number of individuals or entries are correlated, yet the ``average amount'' of correlation is low. 

Because of these properties, we adopt Pufferfish as our privacy definition. To bolster the case for Pufferfish, we provide a general result showing that even if the adversary's belief $\tilde{\theta}$ lies outside the class $\Theta$, the resulting loss in privacy is low if $\tilde{\theta}$ is close to $\Theta$.

The main challenge in using Pufferfish is a lack of suitable mechanisms. While mechanisms are known for specific instantiations~\cite{KM12, DKM14, BDP}, there is no mechanism for general Pufferfish. In this paper, we provide the first mechanism, called the Wasserstein Mechanism, that can be adopted to any general Pufferfish instantiation. Since this mechanism may be computationally inefficient, we consider the case when correlation between variables can be described by a Bayesian network, and the goal is to hide the private value of each variable. We provide a second mechanism, called the Markov Quilt Mechanism, that can exploit properties of the Bayesian network to reduce the computational complexity. As a case study, we derive a simplified version of the mechanism for the physical activity measurement example. We provide privacy and utility guarantees, establish composition properties, and finally demonstrate the practical applicability of the mechanism through experimental evaluation on synthetic as well as real data. Specifically, our contributions are as follows:

\begin{itemize}\itemsep=1pt\parskip=1pt

\item We establish that when we guarantee Pufferfish privacy with respect to a distribution class $\Theta$, but the adversary's belief $\tilde{\theta}$ lies outside $\Theta$, the resulting loss in privacy is small when $\tilde{\theta}$ is close to $\Theta$. 

\item We provide the first mechanism that applies to any Pufferfish instantiation and is a generalization of the Laplace mechanism for differential privacy. We call this the Wasserstein Mechanism.

\item Since the above mechanism may be computationally inefficient, we provide a more efficient mechanism called the Markov Quilt Mechanism when correlation between entries is described by a Bayesian Network. 

\item We show that under certain conditions, applying the Markov Quilt Mechanism multiple times over the same database leads to a gracefully decaying privacy parameter. This makes the mechanism particularly attractive as Pufferfish privacy does not always compose~\cite{KM12}.

\item We derive a simplified and computationally efficient version of this mechanism for time series applications such as physical activity monitoring.

\item Finally, we provide an experimental comparison between Markov Quilt Mechanism and standard baselines as well as some concurrent work~\cite{ghosh2016inferential}; experiments are performed on simulated data, a moderately-sized real dataset ($\approx 10,000$ observations per person) on physical activity, and a large dataset (over $1$ million observations) on power consumption.
\end{itemize} 

\subsection{Related Work} 

There is a body of work on differentially private mechanisms~\cite{NRS07, MT07, CMS11, CHS14, DL09} -- see surveys~\cite{SC13, DR14}. As we explain earlier, differential privacy is not the right formalism for the kind of applications we consider. A related framework is coupled-worlds privacy~\cite{BS13}; while it can take data distributions into account through a distribution class $\Theta$, it requires that the participation of a single individual or entry does not make a difference, and is not suitable for our applications. We remark that while mechanisms for {\em{specific}} coupled-worlds privacy frameworks exist, there is also no {\em{generic}} coupled-worlds privacy mechanism.

Our work instead uses Pufferfish, a recent generalization of differential privacy~\cite{KM12}. \cite{KM12, DKM14} provide some specific instances of Pufferfish frameworks along with associated privacy mechanisms; they do not provide a mechanism that applies to any Pufferfish instance, and their examples do not apply to Bayesian networks. \cite{BDP} uses a modification of Pufferfish, where instead of a distribution class $\Theta$, they consider a single generating distribution $\theta$. They consider the specific case where correlations can be modeled by Gaussian Markov Random Fields, and thus their work also does not apply to our physical activity monitoring example. \cite{mittal} designs Pufferfish privacy mechanisms for distribution classes that include Markov Chains. Their mechanism adds noise proportional to a parameter $\rho$ that measures correlation between entries. However, they do not specify how to calculate $\rho$ as a function of the distribution class $\Theta$, and as a result, their mechanism cannot be implemented when only $\Theta$ is known. \cite{Xiong} releases time-varying location trajectories under differential privacy while accounting for temporal correlations using a Hidden Markov Model with publicly known parameters. Finally, \cite{germans} apply Pufferfish to smart-meter data; instead of directly modeling correlation through Markov models, they add noise to the wavelet coefficients of the time series corresponding to different frequencies. 

In concurrent work,~\cite{ghosh2016inferential} provide an alternative algorithm for Pufferfish privacy when data can be written as $X = (X_1, \ldots, X_n)$ and the goal is to hide the value of each $X_i$. Their mechanism is less general than the Wasserstein Mechanism, but applies to a broader class of models than the Markov Quilt Mechanism. In Section~\ref{sec:experiment}, we experimentally compare their method with ours for Markov Chains and show that our method works for a broader range of distribution classes than theirs. 

Finally, there has also been some previous work on differentially private time-series release~\cite{map, xiongseries,RastogiN10}; however, they relate to aggregates over trajectories from a large number of people, unlike trajectories of single subjects as in our work.

\section{The Setting}
\label{sec:model}
\label{sec:setting}

To motivate our privacy framework, we use two applications -- physical activity monitoring of a single subject and flu statistics in a social network. We begin by describing them at a high level, and provide more details in Section~\ref{sec:pfframework}.

\mypara{Example 1: Physical Activity Monitoring.} The database consists of a time-series $X = \{X_1, X_2, \ldots, X_T\}$ where record $X_t$ denotes a discrete physical activity (e.g, running, sitting, etc) of a subject at time $t$. Our goal is to release (an approximate) histogram of activities over a period (say, a week), while preventing an adversary from inferring the subject's activity at a specific time (say, 10:30am on Jan 4).

\mypara{Example 2: Flu Status.} The database consists of a set of records $X = \{ X_1, \ldots, X_n \}$, where record $X_i$, which is $0$ or $1$, represents person $i$'s flu status. The goal is to release (an approximation to) $\sum_i X_i$, the number of infected people, while ensuring privacy against an adversary who wishes to detect whether a particular person Alice in the database has flu. The database includes people who interact socially, and hence the flu statuses are highly correlated.  Additionally, the decision to participate in the database is made at a group-level (for example, workplace-level or school-level), so individuals do not control their participation.

\subsection{The Privacy Framework}
\label{sec:pfframework}

The privacy framework of our choice is Pufferfish~\cite{KM12}, an elegant generalization of differential privacy~\cite{DMNS06}. A Pufferfish framework is instantiated by three parameters -- a set $\calS$ of secrets, a set $\calQ \subseteq \calS \times \calS$ of secret pairs, and a class of data distributions $\Theta$. $\calS$ is the set of possible facts about the database that we might wish to hide, and could refer to a single individual's private data or part thereof. $\calQ$ is the set of secret pairs that we wish to be indistinguishable. Finally, $\Theta$ is a set of distributions that can plausibly generate the data, and controls the amount and nature of correlation. Each $\theta \in \Theta$ represents a belief an adversary may hold about the data, and the goal of the privacy framework is to ensure indistinguishability in the face of these beliefs.

\begin{definition}
A privacy mechanism $M$ is said to be $\epsilon$-Pufferfish private in a framework $(\calS, \calQ, \Theta)$ if for all $\theta \in \Theta$ with $X$ drawn from distribution $\theta$, for all secret pairs $(s_i, s_j) \in \calQ$, and for all $w \in \text{Range}(M)$, we have
\begin{equation} \label{eqn:defpf}
e^{-\epsilon} \leq \frac{P_{M,\theta}(M(X)=w|s_i, \theta)}{P_{M,\theta}(M(X)=w|s_j, \theta)} \leq e^{\epsilon} 
\end{equation}
when $s_i$ and $s_j$ are such that $P(s_i|\theta) \neq 0$, $P(s_j | \theta) \neq 0.$
\end{definition}
Readers familiar with differential privacy will observe that unlike differential privacy, the probability in~\eqref{eqn:defpf} is with respect to the randomized mechanism {\em{and}} the actual data $X$, which is drawn from a $\theta \in \Theta$; to emphasize this, we use the notation $X$ instead of $D$.  

\subsubsection{Properties of Pufferfish} 
An alternative interpretation of~\eqref{eqn:defpf} is that for $X \sim \theta$, $\theta \in \Theta$, for all $(s_i, s_j) \in \calQ$, and for all $w \in \text{Range}(M)$, we have:  
\begin{equation} \label{eqn:postprior}
e^{-\epsilon} \cdot \frac{P(s_i|\theta)}{P(s_j|\theta)} \leq \frac{P(s_i|M(X) = w, \theta)}{P(s_j|M(X) = w, \theta)}  \leq e^{\epsilon} \cdot \frac{P(s_i|\theta)}{P(s_j|\theta)}.
\end{equation}
In other words, knowledge of $M(X)$ does not affect the posterior ratio of the likelihood of $s_i$ and $s_j$, compared to the initial belief. 

\cite{KM12} shows that Differential Privacy is a special case of Pufferfish, where $\calS$ is the set of all facts of the form $s^i_x =$ {\em{Person $i$ has value $x$}} for $i \in \{ 1, \ldots, n\}$ and $x$ in a domain $\calX$, $\calQ$ is set of all pairs $(s^i_x, s^i_z)$ for $x$ and $z$ in $\calX$ with $x\neq z$, and $\Theta$ is the set of all distributions where each individual is distributed independently. Moreover, we cannot have both privacy and utility when $\Theta$ is the set of all distributions~\cite{KM12}.  Consequently, it is essential to select $\Theta$ wisely; if $\Theta$ is too restrictive, then we may not have privacy against legitimate adversaries, and if $\Theta$ is too large, then the resulting privacy mechanisms may have little utility.

Finally, Pufferfish privacy does not always compose~\cite{KM12} -- in the sense that the privacy guarantees may not decay gracefully as more computations are carried out on the same data. However, some of our privacy mechanisms themselves have good composition properties -- see Section~\ref{sec:composition} for more details.

A related framework is Group Differential Privacy~\cite{DR14}, which applies to databases with groups of correlated records. 

\begin{definition}[Group Differential Privacy]\label{def:groupdp}
Let $D$ be a database with $n$ records, and let $\calG = \{ G_1, \ldots, G_k \}$ be a collection of subsets $G_i \subseteq \{ 1, \ldots, n \}$. A privacy mechanism $M$ is said to be $\epsilon$-group differentially private with respect to $\calG$ if for all $G_i$ in $\calG$, for all pairs of databases $D$ and $D'$ which differ in the private values of the individuals in $G_i$, and for all $S \subseteq Range(M)$, we have:
\[ \Pr(M(D) \in S) \leq e^{\epsilon} \Pr(M(D') \in S). \]
\end{definition}

In other words, Definition~\ref{def:groupdp} implies that the participation of each group $G_i$ in the dataset does not make a substantial difference to the outcome of mechanism $M$.

\subsection{Examples}
\label{sec:examplemodel}

\kc{We next show what Pufferfish frameworks would look like for our running examples, and argue why this framework would provide effective privacy.}

We next illustrate how instantiating these examples in Pufferfish would provide effective privacy and utility.

\kc{\mypara{Example 1: Physical Activity Measurement.} Concretely, to model the physical activity example, we set $\calS$ to be the set $\{ s^{t}_a: t = 1,\ldots, T, a \in A \}$, where $A$ is the set of all activities and $s^{t}_a$ means that activity $a$ happens at time $t$.  $\calQ$ is the set of all pairs $(s^{t}_a, s^{t}_b)$ for $a$, $b$ in $A$ and for all $t$. Finally, $\Theta$ is a set of time series models that capture how people switch between activities. For example, a plausible $\Theta$ is a set of Markov Chains $X_1 \rightarrow X_2 \rightarrow \ldots \rightarrow X_T$ where each $X_t$ takes values in $A$. Recall that such a Markov Chain may be fully described by a tuple $(q, P)$ where $q$ is an initial distribution and $P$ is a transition matrix; each $\theta \in \Theta$ thus corresponds to a tuple $(q_{\theta}, P_{\theta})$. \\}

\mypara{Example 1: Physical Activity Measurement.} Let $A$ be the set of all activities -- for example, $\{${\em{walking}}, {\em{running}}, {\em{sitting}}$\}$ -- and let $s^t_a$ denote the event that activity $a$ occurs at time $t$ -- namely, $X_t = a$. In the Pufferfish instantiation, we set $\calS$ as $\{ s^{t}_a: t = 1,\ldots, T, a \in A \}$ -- so the activity at each time $t$ is a secret. $\calQ$ is the set of all pairs $(s^{t}_a, s^{t}_b)$ for $a$, $b$ in $A$ and for all $t$; in other words, the adversary cannot distinguish whether the subject is engaging in activity $a$ or $b$ at any time $t$ for all pairs $a$ and $b$. Finally, $\Theta$ is a set of time series models that capture how people switch between activities. A plausible modeling decision is to restrict $\Theta$ to be a set of Markov Chains $X_1 \rightarrow X_2 \rightarrow \ldots \rightarrow X_T$ where each state $X_t$ is an activity in $A$. Each such Markov Chain can be described by an initial distribution $q$ and a transition matrix $P$. For example, when there are three activities, $\Theta$ can be the set:
\[ \Bigg{\{} \left( \begin{bmatrix} 1 \\ 0 \\ 0 \end{bmatrix}, \begin{bmatrix} 0.1 & 0.5 & 0.4 \\ 0.4 & 0.3 & 0.3 \\ 0.3 & 0.3 & 0.4  \end{bmatrix} \right),
 \left( \begin{bmatrix} 0 \\ 1 \\ 0 \end{bmatrix}, \begin{bmatrix} 0.9 & 0.1 & 0.0 \\ 0.0 & 0.9 & 0.1 \\ 0.1 & 0.0 & 0.9  \end{bmatrix} \right)  \Bigg{\}} .\]

In this example, 
\begin{itemize}\itemsep=1pt\parskip=1pt
\item Differential privacy does not directly apply since we have a single person's data.
\item Entry differential privacy~\cite{HR13, KM12} and coupled worlds privacy~\cite{BS13} add noise with standard deviation $\approx 1/\epsilon$ to each bin of the activity histogram. However, this does not erase the evidence of an activity at time $t$. Moreover, an activity record does not have its agency, and hence its not enough to hide its participation.
\item As all entries are correlated, group differential privacy adds noise with standard deviation $\approx T/\epsilon$, where $T$ is the length of the chain; this destroys all utility.
\end{itemize}
In contrast, in Section~\ref{sec:bn} we show an $\epsilon$-Pufferfish mechanism which will add noise approximately equal to the mixing time over $\epsilon$, and thus offer both privacy and utility for rapidly mixing chains.

\kc{\mypara{Example 2: Flu Status over Social Network.} Here, $\calS$ is the set $\{ s^{i}_0, s^{i}_1: i = 1, \ldots, n\}$, where $s^{i}_j$ means person $i$ has disease status $j$ ($0$ or $1$).  $\calQ$ is $\{ (s^{i}_0, s^{i}_1) : i = 1, \ldots, n \}$, which means we wish the flu status of each person to be hidden. $\Theta$ is a set of models that describe the spread of flu; a plausible $\Theta$ is a graph $G = (X, E)$ that describes personal interactions, and a set of probability distributions over the connected components of $G$ that describes how flu spreads among interacting people.
 We assume that the decision to participate or not is made at a group-level.\\}

\mypara{Example 2: Flu Status over Social Network.} Let $s^i_0$ (resp. $s^i_1$) denote the event that person $i$'s flu status is $0$ (resp. $1$). In the Pufferfish instantiation, we let $\calS$ be the set $\{ s^{i}_0, s^{i}_1: i = 1, \ldots, n\}$ -- so the flu status of each individual is a secret.  $\calQ$ is $\{ (s^{i}_0, s^{i}_1) : i = 1, \ldots, n \}$ -- so the adversary cannot tell if each person $i$ has flu or not. $\Theta$ is a set of models that describe the spread of flu; a plausible $\theta \in \Theta$ is a tuple $(G_{\theta}, p_{\theta})$ where $G_{\theta} = (X, E)$ is a graph of personal interactions, and $p_{\theta}$ is a probability distributions over the connected components of $G_{\theta}$. As a concrete example, $G_{\theta}$ could be an union of cliques $C_1, \ldots, C_k$, and $p_{\theta}$ could be the following distribution on the number $N$ of infected people in each clique $C_i$: $\Pr(N = j) =  e^{2 j} / \sum_{i=0}^{|C_i|} e^{2 i}, j = 0, \ldots, |C_i|.$

Similarly, for the flu status example, we have:
\begin{itemize}\itemsep=1pt\parskip=1pt
\item Both differential privacy and coupled worlds privacy add noise with standard deviation $\approx 1/\epsilon$ to the number of infected people. This will hide whether an Alice participates in the data, but if flu is contagious, then it is not enough to hide evidence of Alice's flu status.  Note that unlike differential privacy, as the decision to participate is made at group level, Alice has no agency over whether she participates, and hence we cannot argue it is enough to hide her participation.

\item Group differential privacy will add noise proportional to the size of the largest connected component; this may result in loss of utility if this component is large, even if the ``average spread'' of flu is low.
\end{itemize}

Again, we can achieve Pufferfish privacy by adding noise proportional to the ``average spread" of flu which may be less noise than group differential privacy. For a concrete numerical example, see Section~\ref{sec:wasserstein}.

\subsection{Guarantee Against Close Adveraries}
\label{sec:guaranteepf}
A natural question is what happens when we offer Pufferfish privacy with respect to a distribution class $\Theta$, but the adversary's belief $\tilde{\theta}$ does not lie in $\Theta$. Our first result is to show that the loss in privacy is not too large if $\tilde{\theta}$ is close to $\Theta$ conditioned on each secret $s_i \in S$, provided closeness is quantified according to a measure called max-divergence. 
\begin{definition}[max-divergence]
Let $p$ and $q$ be two distributions with the same support. The max-divergence $D_{\infty}(p || q)$ between them is defined as: 
\[ D_{\infty}(p || q) = \sup_{x \in \text{support}(p)} \log \frac{p(x)}{q(x)}.\]
\end{definition}

For example, suppose $p$ and $q$ are distributions over $\{1, 2, 3\}$ such that $p$ assigns probabilities $\{1/3, 1/2, 1/6\}$ and $q$ assigns probabilities $\{1/2, 1/4, 1/4\}$ to $1$, $2$ and $3$ respectively. Then, $D_{\infty}(p || q) = \log 2$. Max-divergence belongs to the family of Renyi-divergences~\cite{CoverThomas}, which also includes the popular Kullback-Leibler divergence. 

\mypara{Notation.} Given a belief distribution $\theta$ (which may or may not be in the distribution class $\Theta$), we use the notation $\theta_{|s_i}$ to denote the conditional distribution of $\theta$ given secret $s_i$. 

\begin{theorem} \label{thm:Pufferfishsmoothness}
Let $M$ be a mechanism that is $\epsilon$-Pufferfish private with respect to parameters $(\calS, \calQ, \Theta)$ and suppose that $\Theta$ is a closed set. Suppose an adversary has beliefs represented by a distribution $\tilde{\theta} \notin \Theta$. Then,
\[ e^{-\epsilon - 2\Delta} \leq \frac{P_{M,\theta}(M(X)=w|s_i, \theta)}{P_{M,\theta}(M(X)=w|s_j, \theta)} \leq e^{\epsilon + 2\Delta} \]
where
\[ \Delta = \inf_{\theta \in \Theta} \max_{s_i \in S} \max\left(D_{\infty}( \tilde{\theta}_{|s_i} || \theta_{| s_i}), D_{\infty}( \theta_{|s_i} || \tilde{\theta}_{| s_i})\right) .\]
\end{theorem}

The conditional dependence on $s_i$ cannot be removed unless additional conditions are met. To see this, suppose $\theta$ places probability mass $\{ 0.9, 0.05, 0.05 \}$ and $\tilde{\theta}$ places probability mass $\{0.01, 0.95, 0.04\}$ on databases $D_1$, $D_2$ and $D_3$ respectively. In this case, $\max\{D_{\infty}(\theta || \tilde{\theta}), D_{\infty}(\tilde{\theta}\|\theta)\} = \log 90$. Suppose now that conditioning on $s_i$ tells us that the probability of occurrence of $D_3$ is $0$, but leaves the relative probabilities of $D_1$ and $D_2$ unchanged. Then $\theta_{|s_i}$ places probability mass $\{ 0.9474, 0.0526 \}$ while $\tilde{\theta}_{|s_i}$ places probability mass $\{0.0104, 0.9896\}$ on $D_1$ and $D_2$ respectively. In this case, $\max\{D_{\infty}(\theta_{|s_i} || \tilde{\theta}_{|s_i}), D_{\infty}(\tilde{\theta}_{|s_i} \| \theta_{|s_i})\} = \log 91.0962$ which is larger than $\max\{D_{\infty}(\theta || \tilde{\theta}), D_{\infty}(\tilde{\theta}\|\theta)\}$.

\subsection{Additional Notation}

We conclude this section with some additional definitions and notation that we will use throughout the  paper.

\begin{definition}[Lipschitz]
A query $F: \calX^n \rightarrow \bbR^k$ is said to be $L$-Lipschitz in $L_1$ norm if for all $i$ and $X_i,X'_i\in\calX$,
\[ \| F(X_1, \ldots, X_i, \ldots, X_n) - F(X_1, \ldots, X'_i, \ldots, X_n)\|_1 \leq L. \]
This means that changing one record out of $n$ in a database changes the $L_1$ norm of the query output by at most $L$. 
\end{definition}

For example, if $F(X_1, \ldots, X_n)$ is a (vector-valued) histogram over records $X_1, \ldots, X_n$, then $F$ is $2$-Lipschitz in the $L_1$ norm, as changing a single $X_i$ can affect the count of at most two bins.

We use the notation $\Lap(\sigma)$ to denote a Laplace distribution with mean $0$ and scale parameter $\sigma$. Recall that this distribution has the density function: $h(x) = \frac{1}{2\sigma} e^{-|x|/\sigma}$. Additionally, we use the notation $\card{S}$ to denote the cardinality, or, the number of items in a set $S$.

\section{A General Mechanism}
\label{sec:wasserstein}

While a number of mechanisms for specific Pufferfish instantiations are known~\cite{KM12, DKM14}, there is no mechanism that applies to any general Pufferfish instantiation. We next provide the first such mechanism. Given a database represented by random variable $X$, a Pufferfish instantiation $(\calS, \calQ, \Theta)$, and a query $F$ that maps $X$ into a scalar, we design mechanism $M$ that satisfies $\epsilon$-Pufferfish privacy in this instantiation and approximates $F(X)$.  

Our proposed mechanism is inspired by the Laplace mechanism in differential privacy; the latter adds noise to the result of the query $F$ proportional to the sensitivity, which is the worst case distance between $F(D)$ and $F(D')$ where $D$ and $D'$ are two databases that differ in the value of a single individual. In Pufferfish, the quantities analogous to $D$ and $D'$ are the distributions $P(F(X) | s_i, \theta)$ and $P(F(X) | s_j, \theta)$ for a secret pair $(s_i, s_j)$, and therefore, the added noise should be proportional to the worst case distance between these two distributions according to some metric.

\mypara{Wasserstein Distances.} It turns out that the relevant metric is the $\infty$-Wasserstein distance -- a modification of the Earthmover's distance used in information retrieval and computer vision. 
\begin{definition}[$\infty$-Wasserstein Distance]
Let $\mu$, $\nu$ be two probability distributions on $\bbR$, and let $\Gamma(\mu, \nu)$ be the set of all joint distributions with marginals $\mu$ and $\nu$. The $\infty$-Wasserstein distance between $\mu$ and $\nu$ is defined as:
\begin{align}\label{eqn:Wdefn}
W_\infty (\mu, \nu) = \inf_{\gamma \in \Gamma(\mu, \nu)} \max_{(x,y)\in support(\gamma)} |x - y|.
\end{align}
\end{definition}
Intuitively, each $\gamma \in \Gamma(\mu, \nu)$ is a way to {\em{shift}} probability mass between $\mu$ and $\nu$; the cost of a shift $\gamma$ is: \\ $\max_{(x,y)\in support(\gamma)}|x - y|$, and the cost of the min-cost shift is the $\infty$-Wasserstein distance.  It can be interpreted as the maximum ``distance'' that any probability mass moves while transforming $\mu$ to $\nu$ in the most optimal way possible. A pictorial example is shown in Figure~\ref{fig:ws}.

\begin{figure}[!ht]
    \centering
        \includegraphics[height=0.13\textwidth]{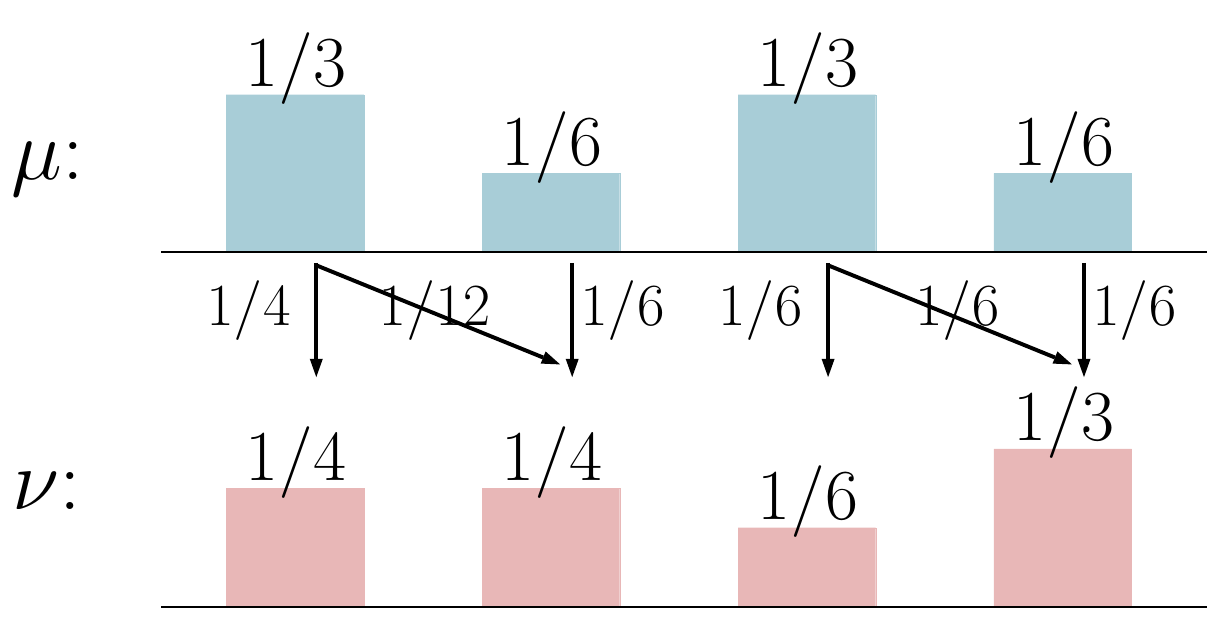}
\caption{An example of $\infty$-Wasserstein distance. The frequency distributions represent two measures $\mu$ and $\nu$, and the arrows represent the joint distribution $\gamma$ that minimizes the right hand side of~\eqref{eqn:Wdefn}. Here, $W_\infty(\mu,\nu) = 1$.}
    \label{fig:ws}
\end{figure}

\subsection{The Wasserstein Mechanism}

The main intuition behind our mechanism is based on this interpretation. Suppose for some $(s_i, s_j)$ and $\theta$, we would like to transform $P(F(X) | s_i, \theta)$ to $P(F(X) | s_j, \theta)$. Then, the maximum ``distance'' that any probability mass moves is 
\[ W_{i, j, \theta} = W_{\infty}\left( P(F(X) | s_i, \theta), P(F(X) | s_j, \theta)\right), \]
and adding Laplace noise with scale $W_{i, j, \theta}/\epsilon$ to $F$ will guarantee that the likelihood ratio of the outputs under $s_i$ and $s_j$ lies in $[e^{-\epsilon}, e^{\epsilon}]$. Iterating over all pairs $(s_i, s_j) \in \calQ$ and all $\theta \in \Theta$ and taking the maximum over $W_{i, j, \theta}$ leads to a mechanism for the entire instantiation $(\calS, \calQ, \Theta)$. 

The full mechanism is described in Algorithm~\ref{alg:wasserstein}, and its privacy properties in Theorem~\ref{thm:wp}. Observe that when Pufferfish reduces to differential privacy, then the corresponding Wasserstein Mechanism reduces to the Laplace mechanism; it is thus a generalization of the Laplace mechanism.

\begin{algorithm}
\caption{Wasserstein Mechanism (Database $D$, query $F$, Pufferfish framework $(\calS, \calQ, \Theta)$, privacy parameter $\epsilon$)}
\label{alg:wasserstein}
\begin{algorithmic}
\FOR {all $(s_i, s_j) \in \calQ$ and all $\theta \in \Theta$ such that $P(s_i | \theta) \neq 0$ and $P(s_j | \theta) \neq 0$}
\STATE{Set $\mu_{i, \theta} = P(F(X) = \cdot | s_i, \theta)$, $\mu_{j, \theta} = P(F(X) = \cdot | s_j, \theta)$. Calculate $W_{\infty}(\mu_{i,\theta}, \mu_{j,\theta})$}
\ENDFOR
 \STATE {Set $W = \sup_{(s_i, s_j) \in \calQ, \theta \in \Theta} W_{\infty}(\mu_{i,\theta}, \mu_{j,\theta}).$}
\STATE{\textbf{return} $F(D) + Z$, where $Z \sim \Lap(\frac{W}{\epsilon})$}
\end{algorithmic}
\end{algorithm}

\kcx{\mypara{Example.} Consider our flu status example with a social interaction graph of size $4$. Suppose that the number $N$ of infected individuals is described by the distribution: for any individual $i$, let $X_i$ be its disease status. Then, $\Pr(N = 0| X_i = 0) = 1/2$, $\Pr(N = j| X_i = 0) = 1/6$ for $j \in [3]$, and moreover, $\Pr(N = j | X_i = 1) = 1/4$ for $j \in [4]$. It is easy to show that the  parameter $W$ in Algorithm \ref{alg:wasserstein} is $2$, and the Wasserstein Mechanism  adds $\Lap(2/\epsilon)$ noise to the number of infected individuals. As all the $X_i$'s are correlated, the global sensitivity mechanism with Group Differential Privacy would add $\Lap(4/\epsilon)$ noise, which gives worse utility.} 

\mypara{Example.} 
Consider a Pufferfish instantiation of the flu status application. Suppose that the database has size $4$, and $\Theta = \{ (G_{\theta}, p_{\theta}) \}$ where $G_{\theta}$ is a clique on $4$ nodes, and $p_\theta$ is the following symmetric joint distribution on the number $N$ of infected individuals:
\setlength\tabcolsep{3pt}
\begin{center}
\begin{tabular}{c||c|c|c|c|c}\hline
$i$			&	$0$		&	$1$		&	$2$		&	$3$		&	$4$		\\ \hline
$P(N=j)$	&	$0.1$	&	$0.15$	&	$0.5$	&	$0.15$	&	$0.1$	\\ \hline
\end{tabular}
\end{center}
From symmetry, this induces the following conditional distributions:

\begin{center}
\begin{tabular}{c||c|c|c|c|c}\hline
$j$				&	$0$		&	$1$		&	$2$		&	$3$		&	$4$		\\ \hline
$P(N=j|X_i=0)$	&	$0.2$	&	$0.225$	&	$0.5$	&	$0.075$	&	$0$		\\ \hline
$P(N=j|X_i=1)$	&	$0$		&	$0.075$	&	$0.5$	&	$0.225$	&	$0.2$	\\ \hline
\end{tabular}
\end{center}

In this case, the  parameter $W$ in Algorithm \ref{alg:wasserstein} is $2$, and the Wasserstein Mechanism will add $\Lap(2/\epsilon)$ noise to the number of infected individuals. As all the $X_i$'s are correlated, the sensitivity mechanism with Group Differential Privacy would add $\Lap(4/\epsilon)$ noise, which gives worse utility.

\subsection{Performance Guarantees}
\begin{theorem}[Wasserstein Privacy]
\label{thm:wp}
The Wasserstein Mechanism provides $\epsilon$-Pufferfish privacy in the framework $(\calS, \calQ, \Theta)$.
\end{theorem}
\mypara{Utility.} Because of the extreme generality of the Pufferfish framework, it is difficult to make general statements about the utility of the Wasserstein mechanism. However, we show that the phenomenon illustrated by the flu status example is quite general. When $X$ is written as $X = (X_1, \ldots, X_n)$ where $X_i$ is the $i$-th individual's private value, and the goal is to keep each individual's value private, the Wasserstein Mechanism for Pufferfish never performs worse than the Laplace mechanism for the corresponding group differential privacy framework. The proof is in the Appendix~\ref{proof:wsvsgdp}.

\begin{theorem}[Wasserstein Utility] \label{thm:wsvsgdp}
Let $(\calS, \calQ, \Theta)$ be a Pufferfish framework, and let $\calG$ be the corresponding group differential privacy framework (so that $\calG$ includes a group $G$ for each set of correlated individuals in $\Theta$). Then, for a $L$-Lipschitz query $F$, the parameter $W$ in the Wasserstein Mechanism 
is less than or equal to the global sensitivity of $F$ in the $\calG$-group differential privacy framework.
\end{theorem}
\section{A Mechanism for Bayesian Networks}
\label{sec:bn}

The Wasserstein Mechanism, while general, may be computationally expensive. We next consider a more restricted setting where the database $X$ can be written as a collection of variables $X = (X_1, \ldots, X_n)$ whose dependence is described by a Bayesian network, and the goal is to keep the value of each $X_i$ private. This setting is still of practical interest, as it can be applied to physical activity monitoring and power consumption data. 

\subsection{The Setting}
\label{sec:bnframework}
\mypara{Bayesian Networks.} Bayesian networks are a class of probabilistic models that are commonly used to model dependencies among random variables or vectors, and include popular models such as Markov Chains and trees. A Bayesian network is described by a set of variables $X = \{ X_1, \ldots, X_n\}$ (where each $X_i$ is a scalar or a vector) and a {\em{directed acyclic graph}} $G = (X, E)$ whose vertices are variables in $X$. Since $G$ is directed acyclic, its edges $E$ induce a parent-child relationship $\parent$ among the nodes $X$. The probabilistic dependence on $X$ induced by the network can be written as:
\[ \Pr(X_1, \ldots, X_n) = \prod_{i} \Pr(X_i | \parent(X_i)). \]
For example, Figure~\ref{fig:bn} shows a Bayesian Network on $4$ variables, whose joint distribution is described by:
\[ \Pr(X_1, X_2, X_3, X_4) = \Pr(X_1) \Pr(X_2 | X_1) \Pr(X_3 | X_1) \Pr(X_4 | X_2, X_3). \]
A node $X_i$ may have more than one parent, and as such these networks can describe complex probabilistic dependencies.
\begin{figure}[!ht]
\centering
\includegraphics[width=0.13\textwidth]{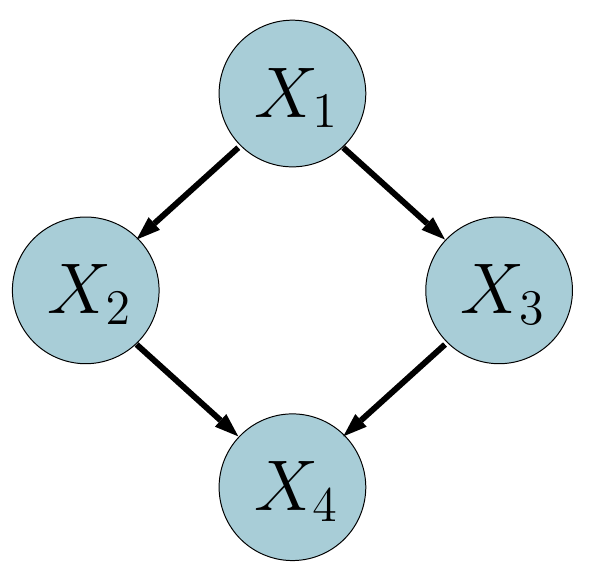}
\caption{A Bayesian Network on four variables.}
\label{fig:bn}
\end{figure}
\vspace{-10pt}

\mypara{The Framework.} Specifically, we assume that the database $X$ can be written as $X = \{ X_1, \ldots, X_n \}$, where each $X_i$ lies in a bounded domain $\calX$. Let $s^i_a$ denote the event that $X_i$ takes value $a$. The set of secrets is $\calS = \{ s^i_a\ : a \in \calX, i \in \kset{n}\}$, and the set of secret pairs is $\calQ = \{ (s^i_a, s^i_b): a, b \in \calX, a\neq b, i \in \kset{n} \}$. We also assume that there is an underlying known Bayesian network $G = (X, E)$ connecting the variables. Each $\theta \in \Theta$ that describes the distribution of the variables is based on this Bayesian network $G$, but may have different parameters. 

For example, the Pufferfish instantiation in Example 1 will fall into this framework, with $n=T$ and $G$ a Markov Chain $X_1 \rightarrow X_2 \rightarrow \ldots X_T$. 

\mypara{Notation.} We use $X$ with a lowercase subscript, for example, $X_i$, to denote a single node in $G$, and $X$ with an uppercase subscript, for example, $X_A$, to denote a set of nodes in $G$. For a set of nodes $X_A$ we use the notation $\card{X_A}$ to denote the number of nodes in $X_A$. 

\subsection{The Markov Quilt Mechanism} 

The main insight behind our mechanism is that if nodes $X_i$ and $X_j$ are ``far apart'' in $G$, then, $X_j$ is largely independent of $X_i$. Thus, to obscure the effect of $X_i$ on the result of a query, it is sufficient to add noise proportional to the number of nodes that are ``local'' to $X_i$ plus a correction term to account for the effect of the distant nodes. The rest of the section will explain how to calculate this correction term, and how to determine how many nodes are local.
 
First, we quantify how much changing the value of a variable $X_i \in X$ can affect a set of variables $X_A \subset X$, where $X_i \notin X_A$ and the dependence is described by a distribution $\theta$ in a class $\Theta$.  To this end, we define the {\em{\effect\ }}of a variable $X_i$ on a set of variables $X_A$ under a distribution class $\Theta$ as follows.
\begin{definition}[\effect] \label{def:effect}
We define the \effect\ of a variable $X_i$ on a set of variables $X_A$ under $\Theta$ as:
\begin{align*}
&e_{\Theta}(X_A | X_i) =\\& \sup_{\theta \in \Theta} \max_{a, b \in \calX}  \max_{x_A \in \calX^{\card{X_A}}} \log \frac{\Pr(X_A = x_A | X_i = a, \theta)}{\Pr(X_A = x_A | X_i = b, \theta)}.
\end{align*}
\end{definition}

Here $\calX$ is the range of any $X_j$. The \effect\ is thus the maximum max-divergence between the distributions $X_A | X_i = a, \theta$ and $X_A | X_i = b, \theta$ where the maximum is taken over any pair $(a, b) \in \calX \times \calX$ and any $\theta \in \Theta$. If $X_A$ and $X_i$ are independent, then the \effect\ of $X_i$ on $X_A$ is $0$, and a large \effect\ means that changing $X_i$ can have a large impact on the distribution of $X_A$.
In a Bayesian network, the \effect\ of any $X_i$ and $X_A$ can be calculated given the probabilistic dependence.

Our mechanism will attempt to find large sets $X_A$ such that $X_i$ has low \effect\ on $X_A$ under $\Theta$. The naive way to do so is through brute force search, which takes time exponential in the size of $G$. We next show how structural properties of the Bayesian network $G$ can be exploited to perform this search more efficiently. 

\mypara{Markov Blankets and Quilts.} For this purpose, we provide a second definition that generalizes the Markov Blanket, a standard notion in probabilistic graphical models~\cite{FriedmanKoller}. The Markov Blanket of a node $X_u$ in a Bayesian network consists of its parents, its children and the other parents of its children, and the rest of the nodes in the network are independent of $X_u$ conditioned on its Markov Blanket. We define its generalization, the Markov Quilt, as follows.
\begin{definition}[Markov Quilt]
\label{def:mquilt}
A set of nodes $X_Q$, $Q \subset \kset{n}$ in a Bayesian network $G = (X, E)$ is a Markov Quilt for a node $X_i$ if the following conditions hold:\\
1. Deleting $X_Q$ partitions $G$ into parts $X_N$ and $X_R$ such that $X = X_N \cup X_Q \cup X_R$ and $X_i \in X_N$.\\
2. For all $x_R \in \calX^{\card{X_R}}$, all $x_Q \in \calX^{\card{X_Q}}$ and for all $a \in \calX$, $P(X_R = x_R | X_Q = x_Q, X_i = a) = P(X_R = x_R | X_Q = x_Q)$.\\
\quad Thus, $X_R$ is independent of $X_i$ conditioned on $X_Q$.
\end{definition}
Intuitively, $X_R$ is a set of ``remote" nodes that are far from $X_i$, and $X_N$ is the set of ``nearby" nodes; $X_N$ and $X_R$ are separated by the Markov Quilt $X_Q$. Observe that unlike Markov Blankets, a node can have many Markov Quilts. Figure \ref{fig:markov blanket quilt} shows an example. We also allow the ``trivial Markov Quilt" with $X_Q = \emptyset$, $X_N = X$ and $X_R = \emptyset$. 

\begin{figure*}[!t]
    \centering
    \begin{subfigure}[b]{0.3395\textwidth}
        \includegraphics[width=\linewidth]{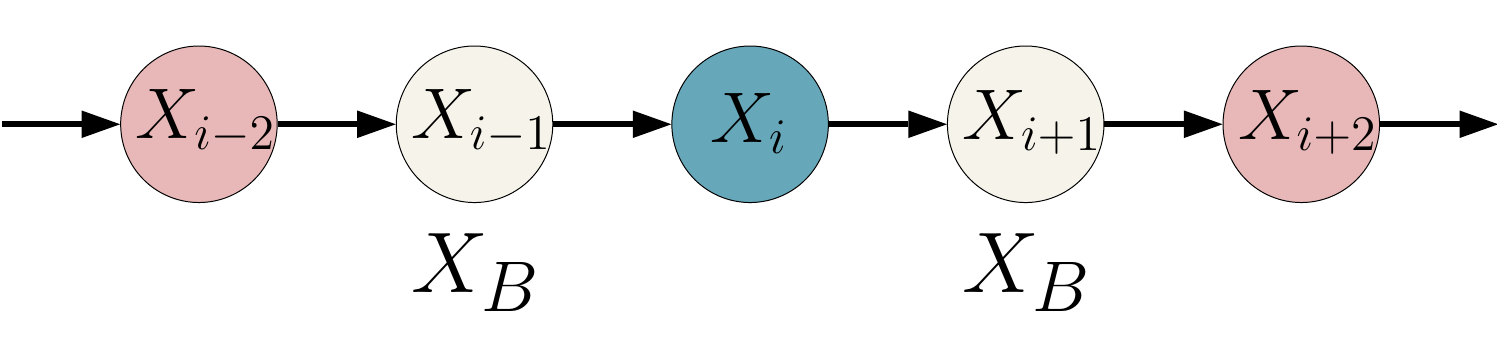}
		\caption{The Markov Blanket of $X_i$ is $X_B=\{X_{i-1}, X_{i+1}\}$}
        \label{fig:markov_blanket}
    \end{subfigure}
    \hfill
    \begin{subfigure}[b]{0.5615\textwidth}
        \includegraphics[width=\linewidth]{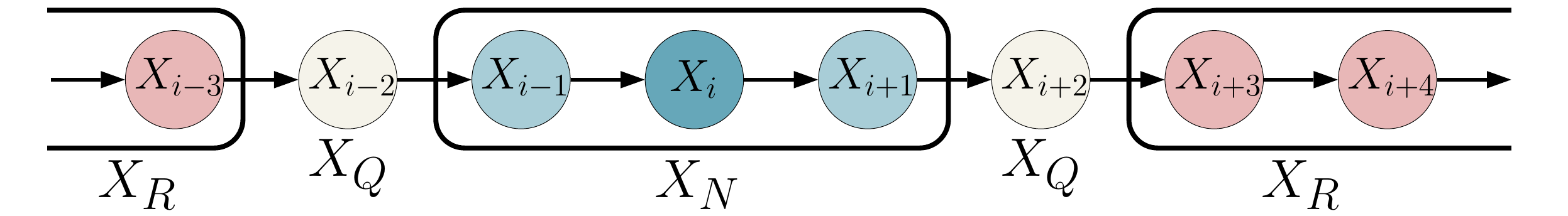}
		\caption{One Markov Quilt of $X_i$ is $X_Q = \{X_{i-2}, X_{i+2}\}$ with corresponding $X_{N} = \{X_{i-1}, X_i, X_{i+1}\}$, $X_{R} = \{\dots,X_{i-3}, X_{i+3},\dots\}$}
        \label{fig:markov_quilt}
    \end{subfigure}
    \caption{Illustration of Markov Blanket and Markov Quilt}
    \label{fig:markov blanket quilt}
\end{figure*}

Suppose we would like to release the result of a $L$-Lipschitz scalar query $F$ while protecting a node $X_i$. If we can find a Markov Quilt $(X_N, X_Q, X_R)$ such that the \effect\ of $X_i$ on $X_Q$ under $\Theta$ is at most $\delta$, then, it is sufficient to add Laplace noise to $F$ with scale parameter $L\cdot\card{X_N}/(\epsilon - \delta)$. This motivates the following mechanism, which we call the Markov Quilt Mechanism. To protect $X_i$, we search over a set $S_{Q, i}$ of Markov Quilts for $X_i$ and pick the one which requires adding the least amount of noise. We then iterate over all $X_i$ and add to $F$ the maximum amount of noise needed to protect any $X_i$; this ensures that the private values of {\em{all nodes}} are protected. Details are presented in Algorithm \ref{alg:mqm}, and Theorem~\ref{thm:mq} establishes its privacy properties.

\begin{algorithm}
\caption{Markov Quilt Mechanism (Database $D$, $L$-Lipschitz query $F$, Pufferfish instantiation $(\calS, \calQ, \Theta)$, privacy parameter $\epsilon$, Markov Quilts set $S_{Q,i}$ for each $X_i$)}
\label{alg:mqm}
\begin{algorithmic}
\FOR {all $X_i$}
 \FOR {all Markov Quilts $X_Q$ (with $X_N, X_R$) in $S_{Q, i}$}
	\IF {\effect\ $e_{\Theta}(X_Q | X_i) < \epsilon$} 
		\STATE{$\sigma(X_Q) = \frac{\card{X_N}}{\epsilon - e_{\Theta}(X_Q | X_i)}$ \qquad /*\textbf{score of $X_Q$}*/}
	\ELSE 
		\STATE{$\sigma(X_Q) = \infty$}
	\ENDIF
 \ENDFOR 
\STATE {$\sigma_i = \min_{X_Q \in S_{Q,i}} \sigma(X_Q)$}  
\ENDFOR
\STATE {$\sigma_{\max} = \max_i \sigma_i$}
\STATE{\textbf{return} $F(D) + (L \cdot \sigma_{\max}) \cdot Z$, where $Z \sim \Lap(1)$}
\end{algorithmic}
\end{algorithm}

\mypara{Vector-Valued Functions.} The mechanism can be easily generalized to vector-valued functions. If $F$ is $L$-Lipschitz with respect to $L_1$ norm, then from Proposition 1 of~\cite{DMNS06}, adding noise drawn from $L\cdot \sigma_{\max} \cdot\Lap(1)$ to each coordinate of $F$ guarantees $\epsilon$-Pufferfish privacy.

\begin{theorem}[Markov Quilt Mechanism Privacy]
If $F$ is $L$-Lipschitz, and if each $S_{Q, i}$ contains the trivial quilt 
$X_Q = \emptyset$ (with $X_N = X$, $X_R = \emptyset$)
, then the Markov Quilt Mechanism preserves $\epsilon$-Pufferfish privacy in the instantiation $(\calS, \calQ, \Theta)$ described in Section~\ref{sec:bnframework}. 
\label{thm:mq}
\end{theorem}

\subsection{Composition}
\label{sec:composition}

Unlike differential privacy, Pufferfish privacy does not always compose~\cite{KM12}, in the sense that the privacy parameter may not decay gracefully as the same data (or related data) is used in multiple privacy-preserving computations. We show below that the Markov Quilt Mechanism {\em{does compose}} under certain conditions. We believe that this property makes the mechanism highly attractive.

\begin{theorem}[Sequential Composition]
\label{thm:seqcompose}
Let $\{F_k\}_{k=1}^K$ be a set of Lipschitz queries, $(\calS,\calQ,\Theta)$ be a Pufferfish instantiation as defined in Section~\ref{sec:bnframework}, and $D$ be a database. Given fixed Markov Quilt sets $\{S_{Q,i}\}_{i=1}^n$ for all $X_i$, let $M_k(D)$ denote the Markov Quilt Mechanism that publishes $F_k(D)$ with $\epsilon$-Pufferfish privacy under $(\calS,\calQ,\Theta)$ using Markov Quilt sets $\{S_{Q,i}\}_{i=1}^n$. Then publishing $(M_1(D),\dots,M_K(D))$
guarantees $K\epsilon$-Pufferfish privacy under $(\calS,\calQ,\Theta)$.
\end{theorem}

To prove the theorem, we define active Markov Quilt $X_{Q,i}^*$ for a node $X_i$.
\begin{definition}(Active Markov Quilt)
In an instance of the Markov Quilt Mechanism $M$, we say that a Markov Quilt $X_Q$ (with corresponding $X_N, X_R$) for a node $X_i$ is {\em{active}} if
$X_Q = \arg\min_{X_Q \in S_{Q,i}} \sigma(X_Q)$, and thus $\sigma(X_Q) = \sigma_i$.
We denote this Markov Quilt as $X_{Q,i}^*$.
\end{definition}

For example, consider a Markov Chain of length $T = 3$ with the following initial distribution and transition matrix
\begin{align*}
\left( \begin{bmatrix} 0.8 \\ 0.2 \end{bmatrix}, \begin{bmatrix} 0.9 & 0.1 \\ 0.4 & 0.6 \end{bmatrix} \right).
\end{align*}
Suppose we want to guarantee Pufferfish privacy with $\epsilon = 10$. Consider the middle node $X_2$. The possible Markov Quilts for $X_2$ are $S_{Q,2} = \{\emptyset, \{X_1\}, \{X_3\}, \{X_1,X_3\}\}$, which have \effect\ $0, \log 6, \log 6, \log 36$ and quilt sizes $3, 2, 2, 1$ respectively. The scores therefore are $0.3$, $0.2437$, $0.2437$ and $0.1558$, and the active Markov Quilt for $X_2$ is $X_{Q,2}^* = \{X_1,X_3\}$ since it has the lowest score.

The proof to Theorem~\ref{thm:seqcompose} is in Appendix~\ref{sec:appendix_general_prop_mqm}.
We note that the condition for the Markov Quilt Mechanism to compose linearly is that all $M_k$'s use the same active Markov Quilt. This holds automatically if the privacy parameter $\epsilon$ and Markov Quilt set $S_{Q,i}$ are the same for all $M_k$. If different levels of privacy $\{\epsilon_k\}_{k=1}^K$ are required, we can guarantee $K\max_{k\in\kset{K}} \epsilon_k$-Pufferfish privacy as long as the same $S_{Q,i}$ is used across $M_k$.

\subsection{Case Study: Markov Chains}
\label{sec:mc}

Algorithm \ref{alg:mqm} can still be computationally expensive, especially if the underlying Bayesian network is complex and the set $\Theta$ is large and unstructured. In this section, we show that when the underlying graph is a Markov Chain, the mechanism can be made more efficient.

\mypara{The Setting.} We use the same setting as Example 1 that was described in detail in Section~\ref{sec:examplemodel}. We assume that there are $k$ activities so that the state space $\calX = \kset{k}$. Thus each $\theta \in \Theta$ corresponds to a tuple $(q_{\theta}, P_{\theta})$ where $q_{\theta}$ is a $k \times 1$ vector that describes the distribution of the first state and $P_{\theta}$ is a $k \times k$ transition matrix.

\mypara{A Running Example.} As a concrete running example in this section, consider a database of $T = 100$ records represented by a Markov Chain $X_1 \rightarrow \ldots \rightarrow X_{100}$ where each state $X_i$ can take $k = 2$ values, $0$ and $1$. Also let:
\[ \Theta = \bigg\{
\theta_1 = \left( \begin{bmatrix} 1 \\ 0 \end{bmatrix}, \begin{bmatrix} 0.9 & 0.1 \\ 0.4 & 0.6 \end{bmatrix} \right),
\theta_2 = \left( \begin{bmatrix} 0.9 \\ 0.1 \end{bmatrix}, \begin{bmatrix} 0.8 & 0.2 \\ 0.3 & 0.7 \end{bmatrix} \right)
     \bigg\} \]
be the set of distributions. Suppose we set $\epsilon=1$.

\mypara{Sources of Inefficiency.} There are three potential sources of computational inefficiency in the Markov Quilt Mechanism. First, searching over all Markov Quilts in a set is inefficient if there are a large number of quilts. Second, calculating the max-influence of a fixed quilt is expensive if the probabilistic dependence between the variables is complex, and finally, searching over all $\theta \in \Theta$ is inefficient if $\Theta$ is large and unstructured. 

We show below how to mitigate the computational complexity of all three sources for Markov Chains. First we show how to exploit structural information to improve the computational efficiency of the first two. Next we propose an approximation using tools from Markov Chain theory that will considerably improve the computational efficiency of searching over all $\theta \in \Theta$. The improvement preserves $\epsilon$-Pufferfish privacy, but might come at the expense of some additional loss in accuracy.

\subsubsection{Exploiting Structural Information}

First, we show how to exploit structural information to improve the running time.

\mypara{Bounding the Number of Markov Quilts.} Technically, the number of Markov Quilts for a node $X_i$ in a Markov Chain of length $T$ is exponential in $T$. Fortunately however, Lemma~\ref{lem:mqmcminimal} shows that for the purpose of finding the quilt with the lowest score for any $X_i$, it is sufficient to search over only $O(T^2)$ quilts. The proof is in the Appendix.

\begin{lemma}\label{lem:mqmcminimal}
In the setting of Section~\ref{sec:mc}, let the set of Markov Quilts $S_{Q, i}$ be as follows:
\begin{align}\label{eqn:minimal quilt}
S_{Q, i} = \{  \{ X_{i-a}, X_{i+b} \}, \{ X_{i-a} \}, \{X_{i+b}\}, \emptyset \nonumber\\
 | 1 \leq a \leq i - 1, 1 \leq b \leq T - i \}
\end{align}
Consider any Markov Quilt $(X_N, X_Q, X_R)$ for $X_i$ that may or may not lie in $S_{Q_i}$. Then, the score of this quilt is greater than or equal to the score of some $(X_{N'}, {X}_{Q'}, X_{R'})$ in $S_{Q, i}$. 
\end{lemma}

\mypara{Calculating the \effect.} Consider the \effect\ of a set of variables $X_Q$ on a variable $X_i$ under a fixed $\theta \in \Theta$.  Calculating it may be expensive if the probabilistic dependency among the variables in $X_Q$ and $X_i$ is complex under $\theta$. We show below that for Markov Chains, this \effect\ may be computed relatively efficiently.

Suppose $X_Q = \{ X_{i-a}, X_{i+b} \}$, and recall that a state can take values in $\calX = \{ 1, \ldots,k\}$. Then, we can write:
\begin{align}\label{eqn:maxinfluence_exact}
&e_{\{\theta\}}(X_Q | X_i) \nonumber \\
= &	\max_{x, x', x_{i-a}, x_{i+b} \in \calX} \log \frac{\Pr(X_{i-a}=x_{i-a}, X_{i+b}=x_{i+b} | X_i = x, \theta)}{\Pr(X_{i-a}=x_{i-a}, X_{i+b}=x_{i+b} | X_i = x', \theta)}\nonumber\\
=&	\max_{x, x'\in \calX} \left(\log \frac{\Pr(X_i = x', \theta)}{\Pr(X_i = x, \theta)}\right. \nonumber\\
&+\max_{x_{i+b}\in\calX} \log \frac{\Pr(X_{i+b}=x_{i+b} | X_i = x, \theta)}{\Pr(X_{i+b}=x_{i+b} | X_i = x', \theta)} \nonumber\\
&+\left. \max_{x_{i-a}\in\calX} \log \frac{\Pr(X_i = x | X_{i-a}=x_{i-a}, \theta)}{\Pr(X_i = x' | X_{i-a}=x_{i-a}, \theta)}\right).
\end{align}
\eqref{eqn:maxinfluence_exact} allows for efficient computation of the \effect. Given an initial distribution and a transition matrix, for each $x$ and $x'$ pair, the first term in~\eqref{eqn:maxinfluence_exact} can be calculated in $O(ik^2)$ time, and the second and third in $O( (a + b) k^3)$ time each. Taking the maximum over all $x$ and $x'$ involves another $O(k^2)$ operations, leading to a total running time of $O(ik^2 + (a + b) k^3)$. Similar formulas can be obtained for quilts of the form $X_Q = \{X_{i-a}\}$ and $X_Q = \{X_{i+b}\}$. Note for the special case of the trivial Markov Quilt, the max-influence is $0$, which takes $O(1)$ time to compute. 

\mypara{Algorithm \mmqmexact.} Thus, using structural properties of Markov Chains gives a more efficient instantiation of the Markov Quilt Mechanism that we describe in Algorithm~\ref{alg:mqmExact}. 

\begin{algorithm}
\caption{\mmqmexact (Database $D$, $L$-Lipschitz query $F$, $\Theta$, privacy parameter $\epsilon$, maximum Markov Quilt size parameter $\ell$)}
\label{alg:mqmExact}
\begin{algorithmic}
\FOR {all $\theta \in \Theta$}
\FOR {all $X_i$}
 \STATE {$S_{Q,i} = \{\{X_{i-a},X_{i+b}\}: i-a>0, i+b\leq T, a+b<\ell\}$ $\cup$ $\{\{X_{i-a}\}: i-a\geq T-\ell\}$ $\cup$ $\{\{X_{i+b}\}: i+b< \ell\}$ $\cup$ $\{\emptyset\}$ \qquad /*\textbf{all quilts with end-points at distance $\leq \ell$ plus the trivial quilt }*/} 
 \FOR {all Markov Quilts $X_Q \in S_{Q,i}$}
 	\STATE {Calculate $e_{\{\theta\}}(X_Q | X_i)$ from \eqref{eqn:maxinfluence_exact}}
	\IF {$e_{\{\theta\}}(X_Q | X_i) < \epsilon$} 
		\STATE{$\sigma(X_Q) = \frac{\card{X_N}}{\epsilon - e_{\{\theta\}}(X_Q | X_i)}$ \qquad /*\textbf{score of $X_Q$}*/}
	\ELSE 
		\STATE{$\sigma(X_Q) = \infty$}
	\ENDIF
\ENDFOR 
\STATE {$\sigma_i = \min_{X_Q \in S_{Q,i}} \sigma(X_Q)$}  
\ENDFOR
\STATE {$\sigma_{\max}^\theta = \max_i \sigma_i$}
\ENDFOR
\STATE {$\sigma_{\max} = \max_{\theta\in\Theta} \sigma_{\max}^\theta$}
\STATE{\textbf{return} $F(D) + (L \cdot \sigma_{\max}) \cdot Z$, where $Z \sim \Lap(1)$}
\end{algorithmic}%
\end{algorithm}

We remark that to improve efficiency, instead of searching over all $O(T^2)$ Markov Quilts for a node $X_i$, we search only over Markov Quilts whose endpoints lie at a distance $\leq \ell$ from $X_i$. Since there are $O(\ell^2)$ such Markov Quilts, this reduces the running time if $\ell \ll T$. 

Concretely, consider using Algorithm \mmqmexact\ on our running example with $\ell = T$. For $\theta_1$, a search over each $X_i$ gives that $X_{8}$ has the highest score, which is $13.0219$, achieved by Markov Quilt $\{X_{3}, X_{13}\}$. For $\theta_2$, $X_{6}$ has the highest score, $10.6402$, achieved by the Markov Quilt $\{X_{10}\}$. Thus, the algorithm adds noise $Z \sim \Lap(13.0219 \times L)$ to the exact query value.%

\mypara{Running Time Analysis.} A naive implementation of Algorithm~\ref{alg:mqmExact} would run in time $O(T \ell^2 |\Theta| ( \ell k^3 + T k^2))$. However, we can use the following observations to speed up the algorithm considerably.

First, observe that for a fixed $\theta$, we can use dynamic programming to compute and store {\em{all}} the probabilities $P(X_i | \theta)$, $P(X_{i} | X_{i-a})$ and $P(X_{i+b} | X_i)$ together in time $O(T k^3)$. Second, note that once these probabilities are stored, the rest is a matter of maximization; for fixed $X_i$, Markov Quilt $X_Q$ and $\theta$, we can calculate $e_{\{\theta\}}(X_Q | X_i)$ from~\eqref{eqn:maxinfluence_exact} in time $O(k^3)$; iterating over all $X_i$, and all $O(\ell^2)$ Markov Quilts for $X_i$ gives a running time of $O(T k^3 + T \ell^2 k^3)$. Finally, iterating over all $\theta \in \Theta$ gives a final running time of $O( T \ell^2 k^3 |\Theta|)$ which much improves the naive implementation. 

Additional optimizations may be used to improve the efficiency even further. In Appendix~\ref{sec:mqmexactallinit}, we show that if $\Theta$ includes all possible initial distributions for a set of transition matrices, then we can avoid iterating over all initial distributions by a direct optimization procedure.  Finally, another important observation is that when the initial distribution under $\theta$ is the stationery distribution of the Markov Chain -- as can happen when data consists of samples drawn from a Markov process in a stable state, such as, household electricity consumption in a steady state -- then for any $i \in [a, T - b]$, the \effect\ $e_{\{\theta\}}(\{ X_{i - a}, X_{i + b} \} | X_i)$ depends only on $a$ and $b$ and is independent of $i$. This eliminates the need to conduct a separate search for each $i$, and further improves efficiency by a factor of $T$. %

\subsubsection{Approximating the \effect}

Unfortunately, Algorithm~\ref{alg:mqmExact} may still be computationally inefficient when $\Theta$ is large. In this section, we show how to mitigate this effect by computing an {\em{upper bound}} on the max-influence under a set of distributions $\Theta$ in closed form using tools from Markov Chain theory. Note that now we can no longer compute an exact score; however, since we use an {\em{upper bound}} on the score, the resulting mechanism remains $\epsilon$-Pufferfish private. 

\mypara{Definitions from Markov Chain Theory.} We begin with reviewing some definitions and notation from Markov Chain theory that we will need. For a Markov Chain with transition matrix $P_{\theta}$, we use $\pi_{\theta}$ to denote its stationary distribution~\cite{markovchaintheory}. We define the time reversal Markov Chain corresponding to $P_{\theta}$ as follows.

\begin{definition}[time-reversal]
Let $P_{\theta}$ be the transition matrix of a Markov Chain $\theta$. If $\pi_{\theta}$ is the stationery distribution of $\theta$, then, the corresponding time-reversal Markov Chain is defined as the chain with transition matrix $P^*_{\theta}$ where:
\[ P^*_{\theta}(x, y) \pi_{\theta}(x) = P_{\theta}(y, x) \pi_{\theta}(y). \]
\end{definition} 

Intuitively, $P^*_\theta$ is the transition matrix when we run the Markov process described by $P_{\theta}$ backwards from $X_T$ to $X_1$. 

In our running example, for both $\theta_1$ and $\theta_2$, the time-reversal chain has the same transition matrix as the original chain, i.e., $P^*_{\theta_1} = P_{\theta_1}$, $P^*_{\theta_2} = P_{\theta_2}$.

We next define two more parameters of a set $\Theta$ of Markov Chains and show that an upper bound to the \effect\ under $\Theta$ can be written as a function of these two parameters.
First, we define $\pi^{\min}_{\Theta}$ as the minimum probability of any state under the stationary distribution $\pi_\theta$ of any Markov Chain $\theta \in \Theta$. Specifically,
\begin{equation}\label{def:pimin}
\pi^{\min}_{\Theta} = \min_{x \in \calX, \theta \in \Theta} \pi_{\theta}(x).
\end{equation}
In our running example, the stationary distribution of the transition matrix for $\theta_1$ is $[0.8, 0.2]$ and thus $\pi^{\min}_{\{\theta_1\}} = 0.2$; similarly, $\theta_2$ has stationary distribution $[0.6, 0.4]$ and thus $\pi^{\min}_{\{\theta_2\}} = 0.4$. We have $\pi^{\min}_{\Theta} = 0.2$.%

Additionally, we define $g_{\Theta}$ as the minimum eigengap of $P_{\theta} P^*_{\theta}$ for any $\theta \in \Theta$. Formally,
\begin{equation}\label{def:g}
g_{\Theta} = \min_{\theta \in \Theta} \min \{ 1-|\lambda| : P_{\theta} P^*_{\theta} x = \lambda x, |\lambda| < 1 \}.
\end{equation}
In our running example, the eigengap for both $P_{\theta_1} P^*_{\theta_1}$ and $P_{\theta_2} P^*_{\theta_2}$ is $0.75$, and thus we have $g_{\Theta} = 0.75$.%

\mypara{Algorithm \mmqmapprox.} The following Lemma characterizes an upper bound of the \effect\ of a variable $X_i$ on a $X_Q$ under a set $\Theta$ of Markov Chains as a function of $\pi^{\min}_{\Theta}$ and $g_{\Theta}$.

\begin{lemma}\label{lem:mqmc2}
Suppose the Markov Chains induced by each $P_{\theta} \in \Theta$ are irreducible and aperiodic with $\pi_{\Theta} > 0$ and $g_{\Theta} > 0$. If $a, b \geq \frac{2 \log(1/\pi^{\min}_{\Theta})}{g_{\Theta}}$, then the \effect\ $e_{\Theta}(X_Q | X_i)$ of a Markov Quilt $X_Q = \{ X_{i - a}, X_{i + b} \}$ on a node $X_i$ under $\Theta$ is at most: 
\begin{align*} 
\log \left( \frac{ \pi^{\min}_{\Theta} + \exp(-\g_{\Theta} b / 2) }{ \pi^{\min}_{\Theta} - \exp(-\g_{\Theta} b / 2) } \right)
 +  2 \log \left( \frac{ \pi^{\min}_{\Theta} + \exp(-\g_{\Theta} a / 2) }{ \pi^{\min}_{\Theta} - \exp(-\g_{\Theta} a / 2) } \right).
\end{align*}
\end{lemma}

The proof is in the Appendix. Observe that the irreducibility and aperiodicity conditions may be necessary -- without these conditions, the Markov Chain may not mix, and hence we may not be able to offer privacy.

Results similar to Lemma~\ref{lem:mqmc2} can be obtained for Markov Quilts of the form $X_Q = \{ X_{i-a} \}$ or $\{ X_{i + b} \}$ as well. Finally, in the special case that the chains are reversible, a tighter upper bound may be obtained; these are stated in Lemma~\ref{lem:mqmc2_full} in the Appendix.

Lemma~\ref{lem:mqmc2} indicates that when $\Theta$ is parametrized by $\g_{\Theta}$ and $\pi^{\min}_{\Theta}$, (an upper bound on) the score of each $X_Q$ may thus be calculated in $O(1)$ time based on Lemma~\ref{lem:mqmc2}. This gives rise to Algorithm~\ref{alg:mqmApprox}. Like Algorithm~\ref{alg:mqmExact}, we can again improve efficiency by confining our search to Markov Quilts where the local set $X_N$ has size at most $\ell$.%

\mypara{Running Time Analysis.}
A naive analysis shows that Algorithm~\ref{alg:mqmApprox} has a worst case running time of $O(T\ell^2)$ -- $T$ iterations to go over all $X_i$, $O(\ell^2)$ Markov Quilts per $X_i$, and $O(1)$ time to calculate an upper bound on the score of each quilt. However, the following Lemma shows that this running time can be improved significantly when $\Theta$ has some nice properties.

\begin{lemma}\label{lem:fastmqm}
Suppose the Markov Chains induced by $\Theta$ are aperiodic and irreducible with $g_{\Theta}, \pi^{\min}_{\Theta} > 0$. Let $a^* = 2\lceil {\log\left(\frac{\exp({\epsilon/6})+1}{\exp({\epsilon/6})-1} \frac{1}{\piT}\right)}/{\gT} \rceil.$
If the length of the chain $T$ is $\geq 8a^*$, then, the optimal Markov Quilt for the middle node $X_{\lceil T/2 \rceil}$ of the chain is of the form $X_Q = \{ X_{\lceil T/2 \rceil - a}, X_{\lceil T/2 \rceil + b} \}$ where $a + b \leq 4 a^*$. Additionally, the maximum score $\sigma_{\max} = \sigma_{\lceil T/2 \rceil}$.
\end{lemma}

Lemma~\ref{lem:fastmqm} implies that for long enough chains, it is sufficient to search over Markov Quilts of length $\ell = 4 a^*$, and only over $X_i = X_{\lceil T/2 \rceil}$; this leads to a running time of $O((a^*)^2)$, which is considerably better and independent of the length of the chain.%

\begin{algorithm}
\caption{\mmqmapprox (Database $D$, $L$-Lipschitz query $F$, $\Theta$ containing Markov chains of length $T$, privacy parameter $\epsilon$, maximum Markov Quilts length $\ell$}
\label{alg:mqmApprox}
\begin{algorithmic}
\FOR {all $X_i$}
 \STATE {$S_{Q,i} = \{\{X_{i-a},X_{i+b}\}: i-a>0, i+b\leq T, a+b<\ell\}$ $\cup$ $\{\{X_{i-a}\}: i-a\geq T-\ell\}$ $\cup$ $\{\{X_{i+b}\}: i+b< \ell\}$ $\cup$ $\{\emptyset\}$ \qquad /*\textbf{all quilts with end-points at distance $\leq \ell$ plus the trivial quilt }*/} 
 \FOR {all Markov Quilts $X_Q$ in $S_{Q, i}$}
	\STATE{Calculate $e_{\Theta}(X_Q | X_i)$ from Lemma~\ref{lem:mqmc2}}
	\IF {\effect\ $e_{\Theta}(X_Q | X_i) < \epsilon$} 
		\STATE{$\sigma(X_Q) = \frac{\card{X_N}}{\epsilon - e_{\Theta}(X_Q | X_i)}$ \qquad /*\textbf{score of $X_Q$}*/}
	\ELSE 
		\STATE{$\sigma(X_Q) = \infty$}
	\ENDIF
 \ENDFOR 
\STATE {$\sigma_i = \min_{X_Q \in S_{Q,i}} \sigma(X_Q)$}  
\ENDFOR
\STATE {$\sigma_{\max} = \max_i \sigma_i$}
\STATE{\textbf{return} $F(D) + (L \cdot \sigma_{\max}) \cdot Z$, where $Z \sim \Lap(1)$}
\end{algorithmic}%
\end{algorithm}

\mypara{Utility.} We conclude with an utility analysis of Algorithm~\ref{alg:mqmApprox}.
\begin{theorem}[Utility Guarantees]
\label{thm:mqm_utility_lemma}
Suppose we apply Algorithm~\ref{alg:mqmApprox} to release an approximation to a $1$-Lipschitz query $F$ of the states in the Pufferfish instantiation in Example 1. If the length $T$ of the chain satisfies: \\ $T \geq 8 \lceil{\log\left(\frac{\exp({\epsilon/6})+1}{\exp({\epsilon/6})-1} \frac{1}{\pi_{\Theta}}\right)}/{\g_{\Theta}}\rceil + 3$, then the Laplace noise added by the Markov Quilt Mechanism has scale parameter $\leq C/\epsilon$ for some positive constant $C$ that depends only on $\Theta$.
\end{theorem}

The proof is in the Appendix. Theorem~\ref{thm:mqm_utility_lemma} implies that the noise added does not grow with $T$ and the {\em{relative accuracy}} improves with more and more observations. A careful examination of the proof also shows that the amount of noise added is an upper bound on the mixing time of the chain. Thus if $\Theta$ consists of rapidly mixing chains, then Algorithm~\ref{alg:mqmApprox} provides both privacy and utility.

\section{Experiments}\label{sec:experiment}

We next demonstrate the practical applicability of the Markov Quilt Mechanism when the underlying Bayesian network is a discrete time homogeneous Markov chain and the goal is to hide the private value of a single time entry -- in short, the setting of Section~\ref{sec:mc}. Our goal in this section is to address the following questions:

\begin{enumerate}\itemsep=1pt\parskip=1pt

\item What is the privacy-utility tradeoff offered by the Markov Quilt Mechanism as a function of the privacy parameter $\epsilon$ and the distribution class $\Theta$?

\item How does this tradeoff compare against existing baselines, such as~\cite{ghosh2016inferential} and Group-differential privacy?

\item What is the accuracy-run time tradeoff offered by the \mmqmapprox\ algorithm as compared with \mmqmexact?

\end{enumerate}

These questions are considered in three different contexts -- (a) a small problem involving a synthetic dataset generated by a two-state Markov Chain, (b) a medium-sized problem involving real physical activity measurement data and a four-state Markov Chain, and (c) a large problem involving real data on power consumption in a single household over time and a fifty-one state Markov Chain.

\subsection{Methodology}
\mypara{Experimental Setup.} Our experiments involve the Pufferfish instantiation of Example 1 described in Section~\ref{sec:examplemodel}. To ensure that results across different chain lengths are comparable, we release a private {\em{relative frequency histogram}} over states of the chain which represents the (approximate) fraction of time spent in each state. This is a vector valued query, and is $\frac{2}{T}$-Lipschitz in its $L_1$-norm.

For our experiments, we consider three values of the privacy parameter $\epsilon$ that are representative of three different privacy regimes -- $0.2$ (high privacy), $1$ (moderate privacy), $5$ (low privacy). All run-times are reported for a desktop with a 3.00 GHz Intel Core i5-3330 CPU and 8GB memory.
 
\mypara{Algorithms.} Our experiments involve four mechanisms that guarantee $\epsilon$-Pufferfish privacy -- \mgroupdp, \minferen, \\
\mmqmapprox\ and \mmqmexact.

{\bf{\mgroupdp}}\ is a simple baseline that assumes that all entries in a {\em{connected}} chain are completely correlated, and therefore adds $\Lap(1/\epsilon)$ noise to each bin. {\bf{\minferen}}\ is the algorithm proposed by~\cite{ghosh2016inferential}, which defines and computes an ``influence matrix'' for each $\theta\in\Theta$. The algorithm applies only when the spectral norm of this matrix is less than $1$, and the standard deviation of noise added increases as the spectral norm approaches $1$. We also use two variants of the Markov Quilt Mechanism -- {\bf{\mmqmexact}} and {\bf{\mmqmapprox}}. 

\begin{figure*}[!t]
    \centering
	\begin{subfigure}[b]{0.325\textwidth}
        \includegraphics[width=\textwidth]{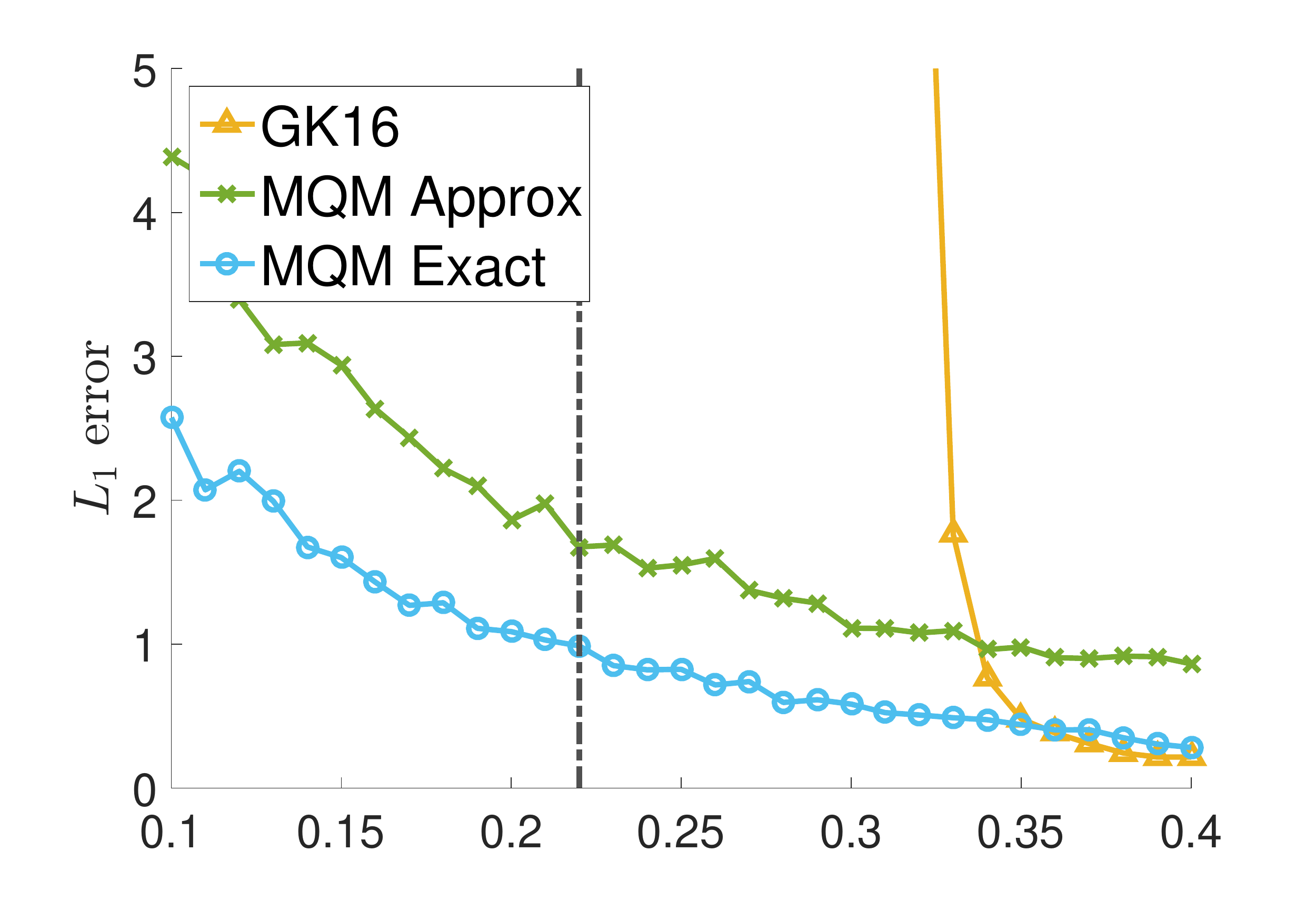}
		\caption{Synthetic binary chain. $\epsilon = 0.2$.}
		\label{fig:syn ep5}
    \end{subfigure}
    \begin{subfigure}[b]{0.325\textwidth}
        \includegraphics[width=\textwidth]{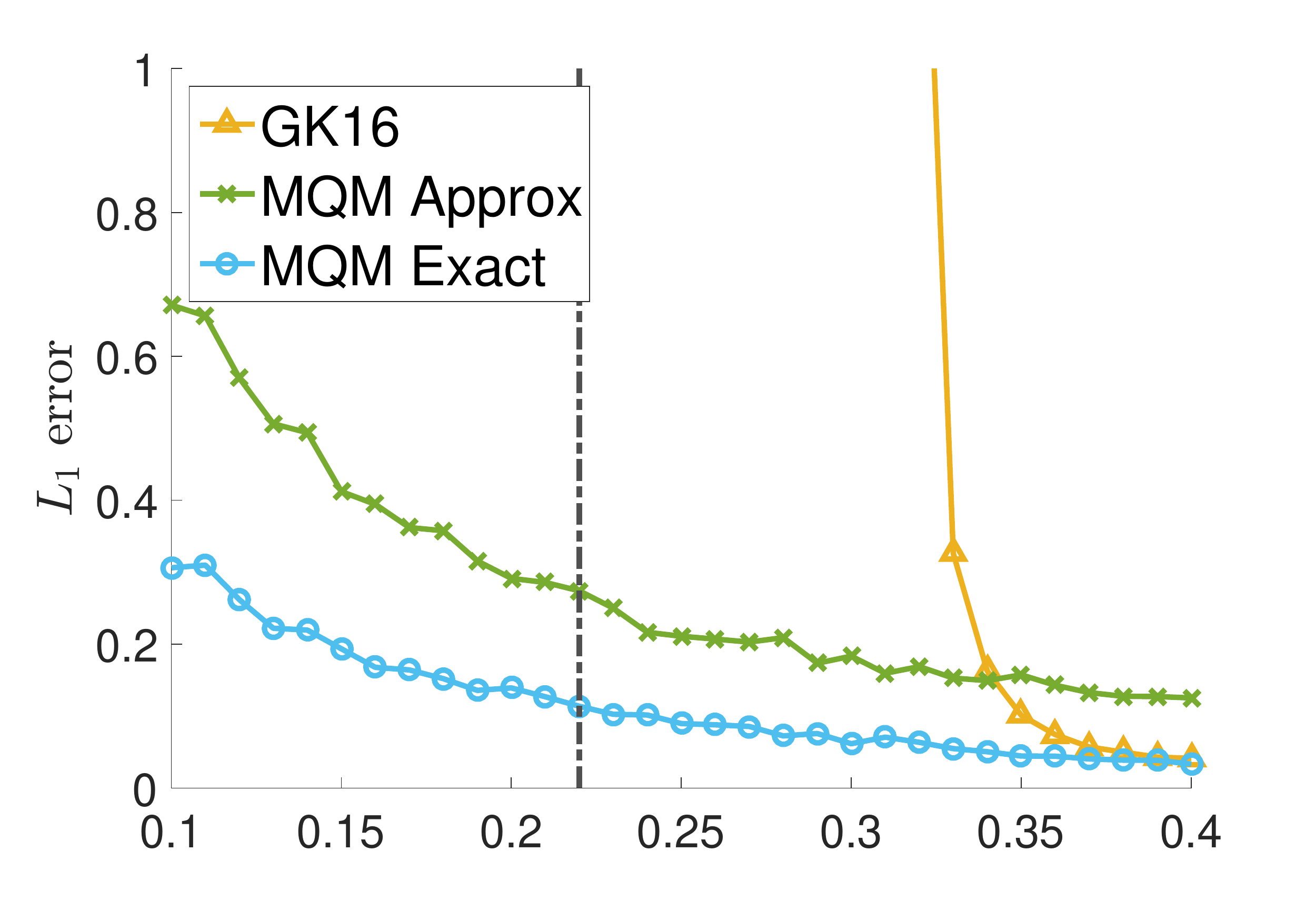}
        \caption{Synthetic binary chain. $\epsilon = 1$.}
        \label{fig:syn ep1}
    \end{subfigure}
    \begin{subfigure}[b]{0.325\textwidth}
        \includegraphics[width=\textwidth]{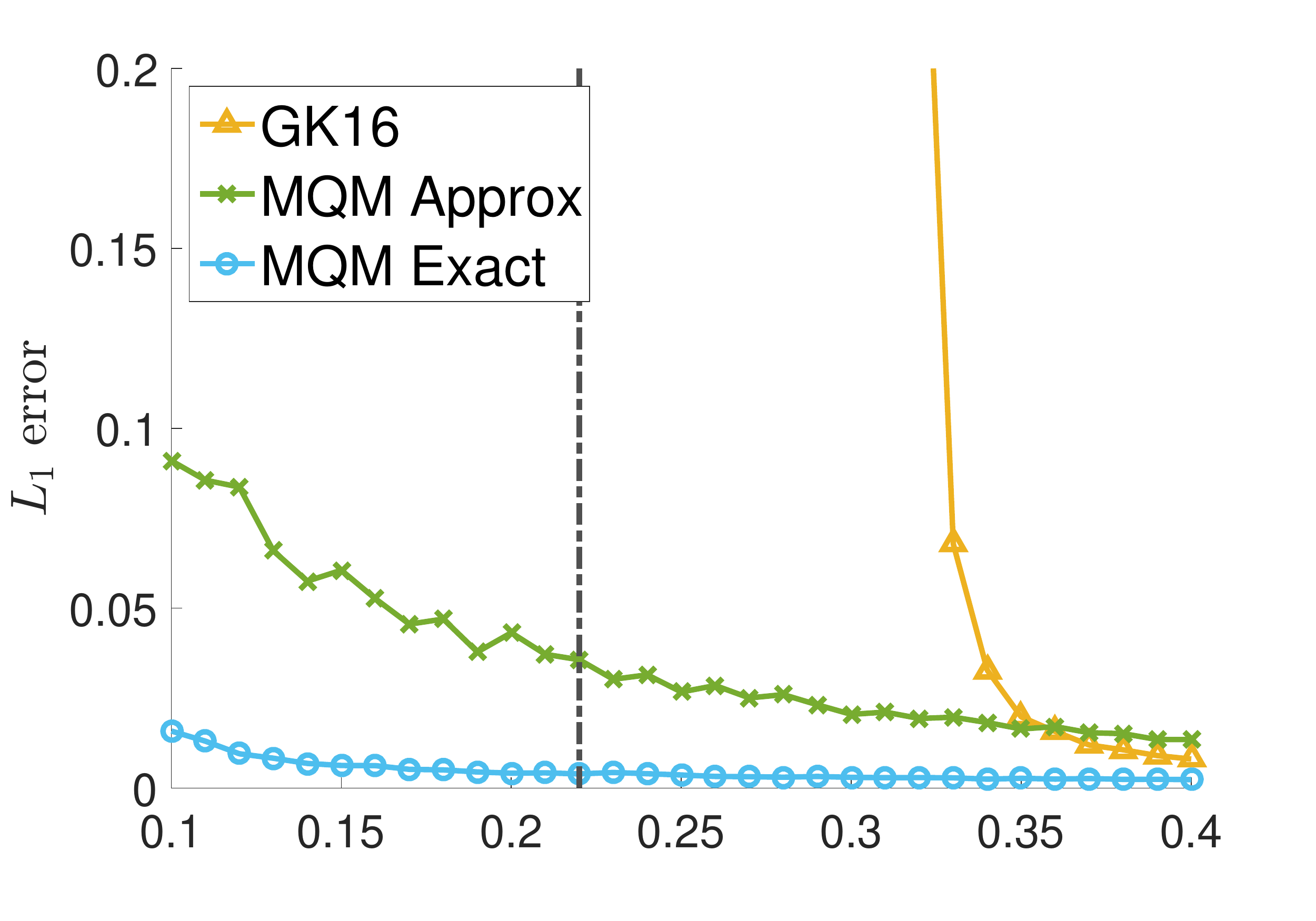}
		\caption{Synthetic binary chain. $\epsilon = 5$.}
		\label{fig:syn ep5}
    \end{subfigure}   
	
   	\begin{subfigure}[b]{0.325\textwidth}
        \includegraphics[width=\textwidth]{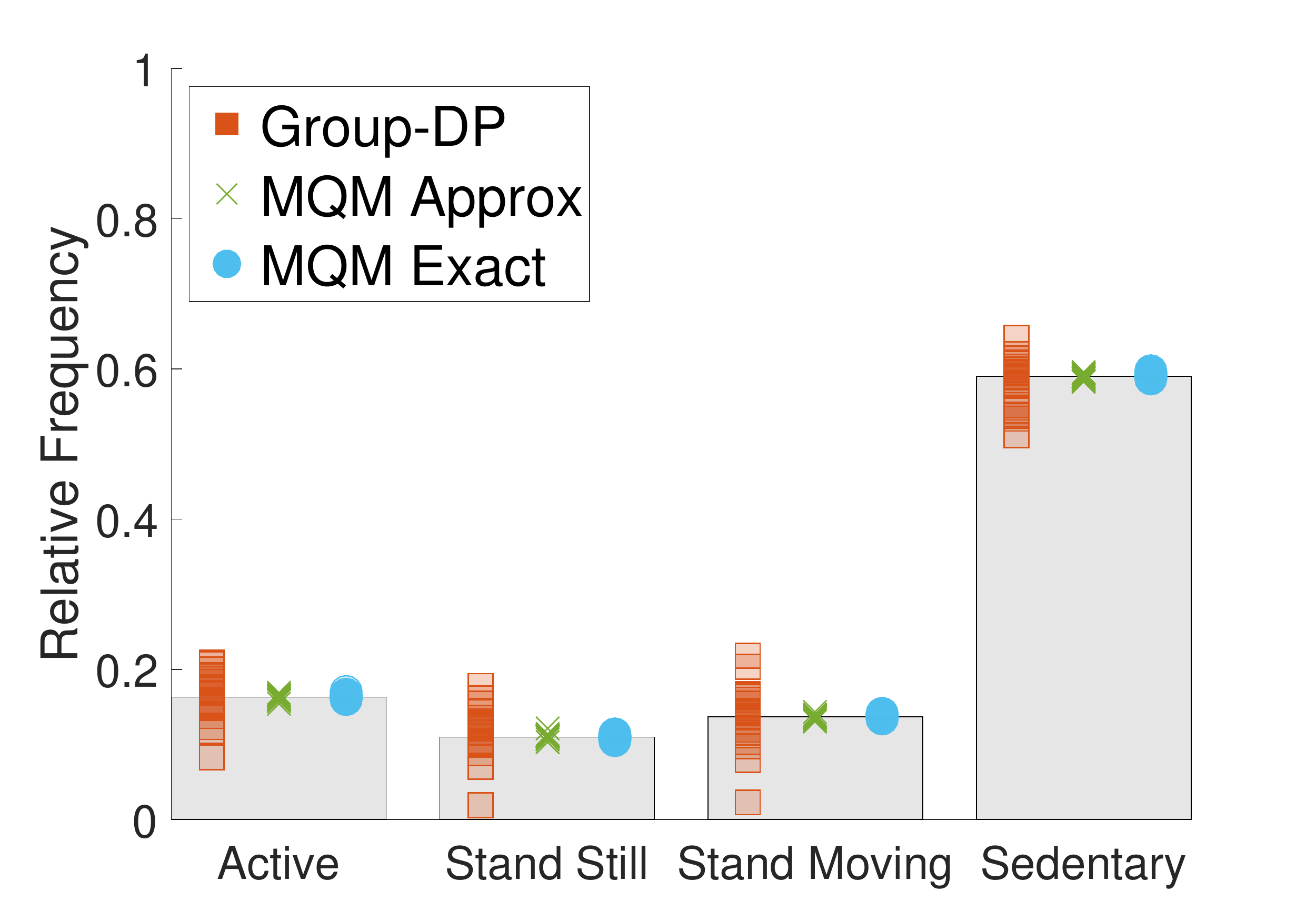}
		\caption{Cyclist aggregate. $\epsilon = 1$.}
    \end{subfigure}    
    	\begin{subfigure}[b]{0.325\textwidth}
        \includegraphics[width=\textwidth]{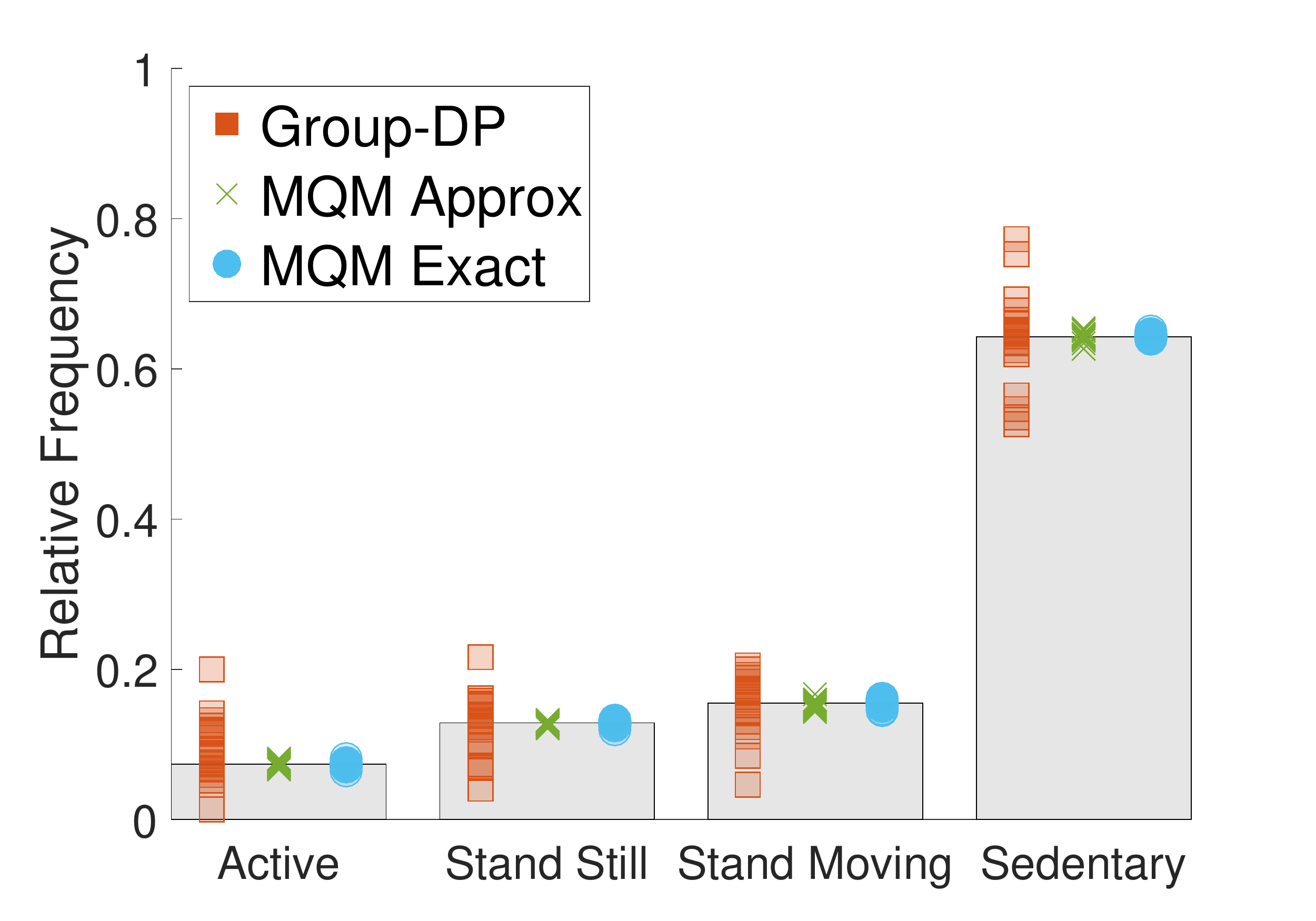}
		\caption{\dataolder\ aggregate. $\epsilon = 1$.}
    \end{subfigure} 
    	\begin{subfigure}[b]{0.325\textwidth}
        \includegraphics[width=\textwidth]{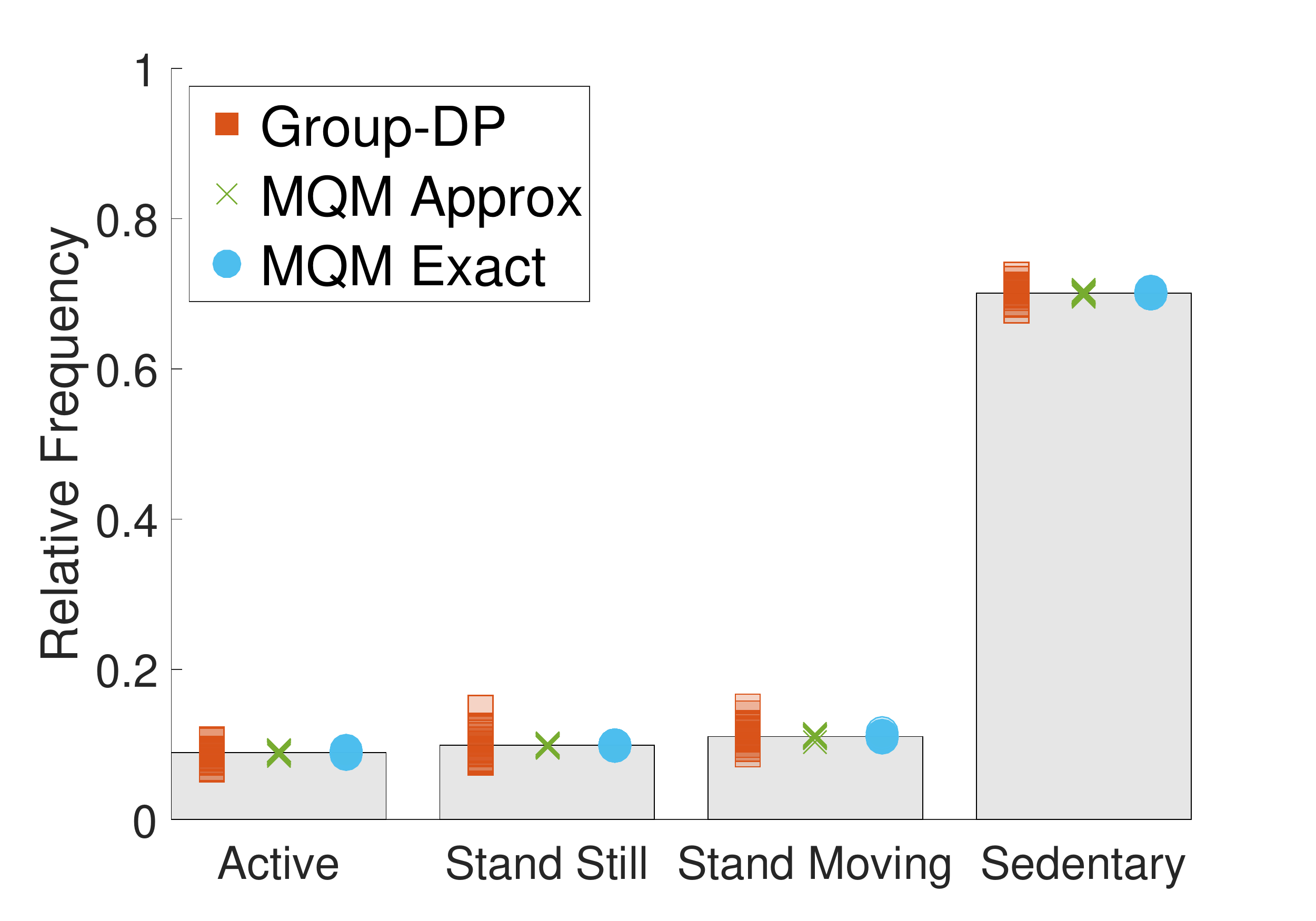}
		\caption{\dataoverw\ aggregate. $\epsilon = 1$.}
    \end{subfigure}

\caption{{\bf{Upper row:}} $L_1$ error of frequency of state $1$ vs. $\alpha$ for $\epsilon=0.2,1,5$ for synthetic data. Recall that $\Theta = [\alpha, 1 - \alpha]$, so $\Theta$ shrinks as we go right. \minferen\ does not apply left of the black dashed vertical line. \mgroupdp\ has error around $5, 1, 0.2$ for $\epsilon = 0.5, 1, 5$ respectively. Reported values are averaged over $500$ random trials. {\bf{Lower row:}} Exact and private aggregated physical activity relative frequency histograms for three groups of participants. Reported values are for $20$ random trials and $\epsilon=1$. Recall \minferen\ does not apply for this problem.}
\label{fig:compare_inferential}
\end{figure*}

\subsection{Simulations}

We first consider synthetic data generated by a binary Markov Chain of length $T = 100$ with states $\{0, 1\}$; the setup is chosen so that all algorithms are computationally tractable when run on reasonable classes $\Theta$. The transition matrix of such a chain is completely determined by two parameters -- $p_0 = \Pr(X_{i+1}=0|X_i=0)$ and $p_1 = \Pr(X_{i+1}=1|X_i=1)$, and its initial distribution by a single parameter $q_0 = \Pr(X_1 = 0)$. Thus a distribution $\theta \in \Theta$ is represented by a tuple $(q_0, p_0, p_1)$, and a distribution class $\Theta$ by a set of such tuples. To allow for effective visualization, we represent the distribution class $\Theta$ by an interval $[\alpha, \beta]$ which means that $\Theta$ includes all transition matrices for which $p_0, p_1 \in [ \alpha, \beta]$ and all initial distribution $q$ in the $2$-dimensional probability simplex. When $0<p_0,p_1<1$, the chain is guaranteed to be aperiodic, irreducible and reversible, and we can use the approximation from Lemma~\ref{lem:mqmc2_full} for \mmqmapprox. We use the optimization procedure described in Appendix~\ref{sec:mqmexactallinit} to improve the efficiency of \mmqmexact. Finally, since the histogram has only two bins, it is sufficient to look at the query $F(X) = \frac{1}{T} \sum_{i=1}^{T} X_i$ which is $\frac{1}{T}$-Lipschitz. 

To generate synthetic data from a family $\Theta = [\alpha, \beta]$, we pick $p_0$ and $p_1$ uniformly from $[\alpha, \beta]$, and an initial state distribution uniformly from the probability simplex. We then generate a state sequence $X_1, \ldots, X_T$ of length $T$ from the corresponding distribution. For ease of presentation, we restrict $\Theta$ to be intervals where $\beta = 1-\alpha$, and vary $\alpha$ from in $0.1$ to $0.4$. We vary $\epsilon$ in $\{0.2, 1, 5\}$, repeat each experiment $500$ times, and report the average error between the actual $F(X)$ and its reported value. For the run-time experiments, we report the average running time of the procedure that computes the scale parameter for the Laplace noise in each algorithm; the average is taken over all $\theta = (p_0, p_1)$ in a grid where $p_0, p_1$ vary in $\{0.1, 0.11, \dots, 0.9\}$.

Figure \ref{fig:compare_inferential} (upper row) shows the accuracy results, and Table~\ref{table:running time} (Column 2) the run-time results for $\epsilon=1$. The values for \mgroupdp\ have high variance and are not plotted in the figure; these values are around $5, 1, 0.2$ respectively. As expected, for a given $\epsilon$, the errors of \minferen, \mmqmapprox\ and \mmqmexact\ decrease as $\alpha$ increases, i.e, the distribution class $\Theta$ becomes narrower. When $\alpha$ is to the left of the black dashed line in the figure, \minferen\ does not apply as the spectral norm of the influence matrix becomes $>1$; the position of this line does not change as a function of $\epsilon$. In contrast, \mmqmapprox\ and \mmqmexact\ still provide privacy and reasonable utility. As expected, \mmqmexact\ is more accurate than \mmqmapprox, but requires higher running time. Thus, the Markov Quilt Mechanism applies to a wider range of distribution families than \minferen; in the region where all mechanisms work, \mmqmapprox\ and \mmqmexact\ perform significantly better than \minferen\ for a range of parameter values, and somewhat worse for the remaining range. 

\subsection{Real Data}
\begin{table*}[ht!]
\begin{center}
  \begin{tabular}{ l || c | c || c | c || c | c  }
    \hline
    \multirow{2}{*}{Algorithm} 		& \multicolumn{2}{|c||}{cyclist} & \multicolumn{2}{|c||}{\dataolder} & \multicolumn{2}{|c}{\dataoverw} \\ \cline{2-7}
    				& Agg & Indi & Agg & Indi & Agg & Indi \\\hline
    \mdp			& 0.2918 & N/A 		& 0.8746	& N/A 		& 0.4763 & N/A\\
    \mgroupdp 		& 0.0834 & 2.3157 	& 0.1138	& 1.7860 	& 0.0458 & 1.1492\\
    \minferen		& N/A 	 & N/A 		& N/A 	& N/A 	& N/A 		& N/A\\
    \mmqmapprox 	& 0.0107 & 0.6319 	& 0.0156 	& 0.2790 	& 0.0048 & 0.1967	\\
    \mmqmexact		& 0.0074 & 0.4077	& 0.0098	& 0.1742	& 0.0033 & 0.1316	\\
    \hline
  \end{tabular}
\end{center}
\vspace{-10pt}
\caption{$L_1$ error of the relative frequency histograms for individual and aggregate tasks for physical activity of three participant groups. $\epsilon=1$. Reported values are averaged over $20$ random trials.}
\label{table:experiment activity}
\vspace{-5pt}
\end{table*}

\begin{table*}[!ht]
\begin{center}
\begin{tabular}{ l || c || c | c | c || c}
\hline
Algorithm 	& Synthetic 						& cyclist 	& \dataolder & \dataoverw & electricity power\\ \hline
\minferen	& $6.3589\times 10^{-4}$ 	 		& N/A 		& N/A		 & N/A 		  & N/A 		\\
\mmqmapprox & $1.8458\times 10^{-4}$  			& 0.0064 	& 0.0060	 & 0.0028	  & 0.0567	\\
\mmqmexact	& $7.6794\times 10^{-4}$ 			& 1.5186 	& 1.2786	 & 0.6299	  & 282.2273			\\
\hline
\end{tabular}
\end{center}
\vspace{-10pt}
\caption{Running time (in seconds) of an optimized algorithm that calculates the scale parameter of the Laplace noise (averaged over $5$ runs). $\epsilon=1$.}
\label{table:running time}
\vspace{-15pt}
\end{table*}

We next apply our algorithm to two real datasets on physical activity measurement and power consumption. Since these are relatively large problems with large state-spaces, it is extremely difficult to search over all Markov Chains in a class $\Theta$, and both \minferen\ as well as \mmqmexact\ become highly computation intensive. To ensure a fair comparison, we pick $\Theta$ to be a singleton set $\{ \theta \}$, where $\theta = (q_{\theta}, P_{\theta})$; here $P_{\theta}$ is the transition matrix obtained from the data, and $q_{\theta}$ is its stationary distribution. For \mmqmapprox, we use $\ell$ from Lemma~\ref{lem:fastmqm}, while for \mmqmexact\, we use as $\ell$ the length of the optimal Markov Quilt that was returned by \mmqmapprox.

\subsubsection{Physical Activity Measurement} 

We use an activity dataset provided by~\cite{data_activity_cyclist, data_activity_overweight1, data_activity_overweight2}, which includes monitoring data from a study of daily habits of $40$ cyclists, $16$ \dataolders, and $36$ \dataoverws. The dataset includes four activities -- active, standing still, standing moving and sedentary -- for all three groups of participants,\footnote{For cyclists, data also includes cycling, which we merge with active for ease of analysis and presentation.} and thus the underlying $\theta$ is a four-state Markov Chain. Activities are recorded about every $12$ seconds for $7$ days when the participants are awake, which gives us more than $9,000$ observations per person on an average in each group. To address missing values, we treat gaps of more than $10$ minutes as the starting point of a new independent Markov Chain. Observe that this improves the performance of \mgroupdp, since the noise added is $\Lap(M/T\epsilon)$, where $M$ is the length of the longest chain. For each group of participants, we calculate a single empirical transition matrix $P_{\theta}$ based on the entire group; this $P_{\theta}$ is used in the experiments. 
For this application, we consider two tasks -- aggregate and individual. In the aggregate task, the goal is to publish a private aggregated relative frequency histogram\footnote{Recall that in a relative frequency histogram, we report the number of observations in each bin divided by the total number of observations.} over participants in each group in order to analyze their comparative activity patterns. While in theory this task can be achieved with differential privacy, this gives poor utility as the group sizes are small. In the individual task, we publish the relative frequency histogram for each individual and report the average error across individuals in each group. 

Table~\ref{table:experiment activity} summarizes the $L_1$ errors for the three groups and both tasks for $\epsilon=1$. For the aggregate task, we also report the error of a differentially private release (denoted as DP). Note that the error of the individual task is the average $L_1$ error of each individual in the group. Figure~\ref{fig:compare_inferential} (lower row) presents the exact and private aggregated relative frequency histograms for the three groups. Table~\ref{table:running time} (Columns 2-4) presents the time taken to calculate the scale parameter of the Laplce noise for each group of participants; as this running time depends on both the size of the group as well as the activity patterns, larger groups do not always have higher running times.

The results in Table~\ref{table:experiment activity} show that the utilities of both \mmqmapprox\ and \mmqmexact\ are significantly better than that of \mgroupdp for all datasets and tasks, and are significantly better than DP for the aggregated task. As expected, \mmqmexact\ is better than \mmqmapprox, but has a higher running time. We also find that \minferen\ cannot be applied to any of the tasks, since the spectral norm of the influence matrix is $> 1$; this holds for all $\epsilon$ as the spectral norm condition does not depend on the value of $\epsilon$. Figure~\ref{fig:compare_inferential} (lower row) shows the activity patterns of the different groups: the active time spent by the cyclist group is significantly longer than the other two groups, and the sedentary time spent by \dataoverws\ group is the longest. These patterns are visible from the private histograms published by \mmqmapprox\ and \mmqmexact, but not necessarily from those published by \mgroupdp.

\subsubsection{Electricity Consumption} 

\begin{table}[h]
\begin{center}
  \begin{tabular}{ l || c | c | c }
    \hline
    Algorithm 		& $\epsilon=0.2$ & $\epsilon=1$ & $\epsilon=5$ \\ \hline
    \mgroupdp  		& 516.1555    & 102.8868		& 19.8712\\
    \minferen			& N/A 	 & N/A 		& N/A\\
    \mmqmapprox 		& 0.3369 & 0.0614 	& 0.0113\\
    \mmqmexact		& 0.1298 & 0.0188 	& 0.0022\\
    \hline
  \end{tabular}
\end{center}
\vspace{-10pt}
\caption{$L_1$ error of relative frequency histogram for electricity consumption data. Reported values are averaged over $20$ random trials.}
\label{table:experiment power}
\vspace{-10pt}
\end{table}
We use data on electricity consumption of a single household in the greater Vancouver area provided by \cite{data_electricity}. Power consumption (in Watt) is recorded every $1$ minute for about two years. Missing values in the original recording have been filled in before the dataset is published. We discretize the power values into $51$ intervals, each of length $200$(W), resulting in a Markov chain with $51$ states of length $T\approx 1,000,000$. Our goal again is to publish a private approximation to the relative frequency histogram of power levels.   

Table~\ref{table:experiment power} reports the $L_1$ errors of the four algorithms on electricity power dataset for different $\epsilon$ values, and Table~\ref{table:running time} the times required to calculate the scale parameter. We again find that \minferen\ does not apply as the spectral norm condition is not satisfied. The error of \mgroupdp\ is very large, because the data forms a single Markov chain, and the number of states is relatively large. We find that in spite of the large number of bins, \mmqmapprox\ and \mmqmexact\ have high utility; for example, even for $\epsilon = 0.2$, the {\em{per bin}} error of \mmqmexact\ is about $0.25$ percent, and for $\epsilon = 1$, the {\em{per bin}} error is $0.04$ percent. Finally, while the running time of \mmqmexact\ is an order of magnitude higher than \mmqmapprox, it is still manageable ($< 5$ minutes) for this problem.

\vspace{15pt}
\newpage
\subsection{Discussion}

We now reconsider our initial questions. First, as expected, the experiments show that the utility of \minferen\ and both versions of \mqm\ decreases with decreasing $\epsilon$ and increasing size of $\Theta$. The utility of \mgroupdp\ does not change with $\Theta$ and is not very high; again this is to be expected as \mgroupdp\ depends only on the worst-case correlation. Overall the utility of both versions of the Markov Quilt Mechanism improve for longer chains. 

Comparing with \minferen, we find that for synthetic data, the Markov Quilt Mechanism applies to a much wider range of distribution families; in the region where all mechanisms work, \mmqmapprox\ and \mmqmexact\ perform significantly better than \minferen\ for a range of parameter values, and somewhat worse for the remaining range. For the two real datasets we consider, \minferen\ does not apply.  We suspect this is because the influence matrix in~\cite{ghosh2016inferential} is calculated based on local transitions between successive time intervals (for example, $X_t$ as a function of $X_{t-1}$); instead, the Markov Quilt Mechanism implicitly takes into account transitions across periods (for example, how $X_{t+k}$ as a function of $X_t$). Our experiments imply that this spectral norm condition may be quite restrictive in practical applications. 

Finally, our experiments show that there is indeed a gap between the performance of \mmqmexact\ and \mmqmapprox, as well as their running times, although the running time of \mmqmexact\ still remains manageable for relatively large problems. Based on these results, we recommend using \mmqmexact\ for medium-sized problems where the state space is smaller (and thus computing the \effect\ is easier) but less data is available, and \mmqmapprox\ for larger problems where the state space is larger but there is a lot of data to mitigate the effect of the approximation.

\vspace{-5pt}
\section{Conclusion}

We present a detailed study of how Pufferfish may be applied to achieve privacy in correlated data problems. We establish robustness properties of Pufferfish against adversarial beliefs, and we provide the first mechanism that applies to any Pufferfish instantiation. We provide a more computationally efficient mechanism for Bayesian networks, and establish its composition properties. We derive a version of our mechanism for Markov Chains, and evaluate it experimentally on a small, medium and a large problem on time series data. Our results demonstrate that Pufferfish offers a good solution for privacy in these problems.

We believe that our work is a first step towards a comprehensive study of privacy in correlated data. There are many interesting privacy problems -- such as privacy of users connected into social networks and privacy of spatio-temporal information gathered from sensors. With the proliferation of sensors and ``internet-of-things'' devices, these privacy problems will become increasingly pressing. We believe that an important line of future work is to model these problems in rigorous privacy frameworks such as Pufferfish and design novel mechanisms for these models.

\section{Acknowledgments} We thank Mani Srivastava and Supriyo Chakravarty for introducing us to the physical activity problem and early discussions and anonymous reviewers for feedback. This work was partially supported by NSF under IIS 1253942 and ONR under N00014-16-1-2616.

\pagebreak

\bibliographystyle{abbrv}
\bibliography{ref,privacy}  %

\appendix
\section{Pufferfish Privacy Details}
\begin{proof}(of Theorem~\ref{thm:Pufferfishsmoothness})
Since $\Theta$ is a closed set, for finite $D$, there must exist some distribution in $\Theta$ that achieves $\Delta$. Call this $\theta$. The ratio $\frac{P(M(X=w)|s_i, \ttheta)}{P(M(X)=w|s_j, \ttheta)}$ is equal to:
\begin{align}\label{eqn:pf pufferfishsmoothness1}
\frac{P(M(X)=w|s_i, \ttheta)}{P(M(X)=w|s_i, \theta)}
	\frac{P(M(X)=w|s_j, \theta)}{P(M(X)=w|s_j, \ttheta)}
	\frac{P(M(X)=w|s_i, \theta)}{P(M(X)=w|s_j, \theta)}.
\end{align}
Bu Pufferfish, the last ratio in \eqref{eqn:pf pufferfishsmoothness1} $\in [e^{-\epsilon}, e^\epsilon]$. Since the outcome of $M$ given $X$ is independent of the generation process for $X$, we have
\begin{align*}
&	\frac{P(M(X)=w|s_i, \ttheta)}{P(M(X)=w|s_i, \theta)}
=	\frac{\int_D P(M(D)=w)P(X=D|s_i, \ttheta) d D}{\int_D P(M(D)=w)P(X=D|s_i, \theta) d D} \\
\leq&	\max_{D} \frac{P(X=D|s_i, \ttheta)}{P(X=D|s_i, \theta)}
\leq	e^{D_{\infty}(X | \tilde{\theta},s_i \| X| \theta, s_i)} \leq e^\Delta.
\end{align*}
Similarly, we can show that this ratio is also $\geq e^{-\Delta}$. Applying the same argument to the second ratio in \eqref{eqn:pf pufferfishsmoothness1} along with simple algebra concludes the proof.
\end{proof}
\section{Wasserstein Mechanism Proofs}%

\begin{proof}(Of Theorem~\ref{thm:wp})
Let $(s_i, s_j)$ be secret pair in $\calQ$ such that that $P(s_i | \theta), P(s_j | \theta) > 0$. Let $\mu_{i, \theta}$,$\mu_{j, \theta}$ be defined as in the Wasserstein Mechanism. Let $\gamma^* = \gamma^*(\mu_{i, \theta}, \mu_{j, \theta})$ be the coupling between $\mu_{i, \theta}$ and $\mu_{j, \theta}$ that achieves the $\infty$-Wasserstein distance. 

Let $M$ denote the Wasserstein mechanism. For any $w$, the ratio $\frac{P(M(X) = w | s_i, \theta)}{P(M(X) = w | s_j, \theta)}$ is:
\begin{eqnarray}
& = & \frac{\int_t P(F(X) = t | s_i, \theta) P(Z = w - t) dt }{ \int_s P(F(X) = s | s_j, \theta) P(Z = w - s) ds} \nonumber \\
& = & \frac{\int_t P(F(X) = t | s_i, \theta) e^{-\epsilon|w - t|/ W} dt }{ \int_s P(F(X) = s | s_j, \theta) e^{-\epsilon |w - s|/ W} ds} \nonumber \\
& = & \frac{\int_t \int_{s = t - W}^{t + W} \gamma^*(t, s) e^{-\epsilon|w - t|/W} ds dt }{\int_s \int_{t = s - W}^{s + W} \gamma^*(t, s) e^{-\epsilon|w - s|/W} dt ds }, \label{eqn:ws1}
\end{eqnarray}
where the first step follows from the definition of the Wasserstein mechanism, the second step from properties of the Laplace distribution, and the third step because for all $W \geq W_\infty(\mu_{i,\theta},\mu_{j,\theta})$, we have 
\[ \mu_{i,\theta}(t) = P(F(X) = t | s_i, \theta)  = \int_{s = t - W}^{t+W} \gamma^*(t, s) ds \]

An analogous statement also holds for $\mu_{j, \theta}(s)$. Observe that in the last step of~\eqref{eqn:ws1}, $|s - t| \leq W$ in both the numerator and denominator; therefore we have that the right hand side of~\eqref{eqn:ws1} is at most:
\begin{align} \label{eqn:ws2}
 e^{\epsilon} &\frac{\int_t \int_{s = t - W}^{t + W} \gamma^*(t, s) e^{-\epsilon|w - s|/W} ds dt }{\int_s \int_{t = s - W}^{s + W} \gamma^*(t, s) e^{-\epsilon|w - s|/W} dt ds } \\
& = e^{\epsilon} \frac{\int_t \int_s \gamma^*(t, s) e^{-\epsilon|w - s|/W} ds dt }{\int_s \int_{t} \gamma^*(t, s) e^{-\epsilon|w - s|/W} dt ds} \leq e^{\epsilon},
\end{align}
where the last step follows as $\gamma^*(t, s) = 0$ when $|s - t| > W$. A similar argument shows that $\frac{P(M(X) = w | s_j, \theta)}{P(M(X) = w | s_i, \theta)} \leq e^{\epsilon}$, thus concluding the proof.\end{proof}

\subsection{Comparison with Group DP}
\label{proof:wmgdp}
\label{proof:wsvsgdp}

Consider a group-DP framework parameterized by a set of groups $\calG = (G_1, \ldots, G_k)$. For each $k$, we use $X_{G_k} \subset X$ to denote the records of all individuals in $G_k$, and we define $D_k = \{(x, y) | x, y\text{ differ only by records in } X_{G_k}\}$.

\begin{definition}[Global Sensitivity of Groups]
We define the global sensitivity of a query $F$ with respect to a group $G_k$ as: $\Delta_{G_k} F = \max_{(x, y)\in D_k} |F(x) - F(y)|$. 
The Global Sensitivity of $F$ with respect to an a Group DP framework $\calG$ is defined as: $\Delta_{\calG}F = \max_{k \in \kset{m}}\Delta_{G_k}F.$
\end{definition}

Analogous to differential privacy, adding to the result of query $F$ Laplace noise with scale parameter $\Delta_{\calG}F/\epsilon$ will provide in $\epsilon$-group DP in the framework $\calG$~\cite{DR14}.

We begin by formally defining a group DP framework corresponding to a given Pufferfish framework. Suppose we are given a Pufferfish instantiation $(\calS, \calQ, \Theta)$ where data $X$ can be written as $X = \{X_1, \cdots, X_n\}$, $X_i \in \calX$, the secret set $\calS = \{ s^i_a\ : a \in \calX, i \in \kset{n}\}$ where $s_a^i$ is the event that $X_i = a$, and the set of secret pairs $\calQ = \{ (s^i_a, s^i_b): a, b \in \calX, a\neq b, i \in \kset{n} \}$. In the corresponding Group DP framework $\calG$, $G_1, \ldots, G_m$ is a partition of $\kset{n}$, such that for any $p\neq q$, $X_i$ and $X_j$ are independent {\em{in all $\theta \in \Theta$}} if $i\in G_p, j \in G_q$.

\begin{proof} (Of Theorem~\ref{thm:wsvsgdp}) All we need to prove is $W \leq \Delta_{\calG}F$, where $W$ is the noise parameter in the Wasserstein Mechanism and $\Delta_{\calG}F$ is the global sensitivity of $F$ in the group-DP framework $\calG$.

For any secret pair $(s^i_a, s^i_b) \in \calQ$, let $k$ be such that $i \in G_k$. Let $X_T = X \setminus X_{G_k}$. For all realizations $x_T$ of $X_T$, define $\nu_{i,a,\theta, x_T} =\Pr(F(X) |X_i=a, X_T = x_T, \theta)$ and $\nu_{i,b,\theta, x_T} =\Pr(F(X_S, X_T = x_T) |X_i=b, \theta)$. 

Since any $X_j \in X_T$ is independent of $X_i$, for all realizations $x_T$ of $X_T$, we have $\Pr(X_T=x_T|X_i=a, \theta) = \Pr(X_T = x_T|X_i=b, \theta) = \Pr(X_T = x_T|\theta)$. And we have $\mu_{i,a,\theta} = \sum_{x_T} P(X_T=x_T | \theta) \nu_{i,a,\theta, x_T}$, and similarly, $\mu_{i,b,\theta} = \sum_{x_T} P(X_T=x_T | \theta) \nu_{i,b,\theta, x_T}$. 
As these two probability distributions are mixtures of the $\nu_{i,a,\theta, x_T}$s and $\nu_{i,b,\theta, x_T}$s with the same mixing coefficients, by Lemma \ref{lemma:wdofmixture}, we have: \\
$W_\infty(\mu_{i,a,\theta}, \mu_{i,b,\theta}) \leq \max_{x_T} W_\infty(\nu_{i,a,\theta, x_T}, \nu_{i,b,\theta, x_T}).$

By definition, $\infty$-Wasserstein of two distributions is upper bounded by the range of the union of the two supports, $W_\infty(\nu_{i,a,\theta, x_T}, \nu_{i,b,\theta, x_T}) \leq \max_{x_{G_k}, x'_{G_k}}|F(X_{G_k}=x_{G_k}, X_T=x_T) - F(X_{G_k}= x'_{G_k}, X_T=x_T)|$, which is $\leq \Delta_{G_k} F$. Thus, $W_\infty(\mu_{i,a,\theta}, \mu_{i,b,\theta}) \leq \Delta_{\calG}F.$ Since this holds for all $(s^i_a, s^i_a) \in Q$ and all $\theta \in \Theta$, $W \leq \Delta_{\calG}F.$
\end{proof}

\begin{lemma}\label{lemma:wdofmixture}
Let $\{\mu_i\}_{i=1}^n$, $\{\nu_i\}_{i=1}^n$ be two collections of probability distributions, and let $\{c_i\}_{i=1}^n$ be mixing weights such that $c_i \geq 0, \sum_i c_i = 1$. Let $\mu = \sum_i c_i\mu_i$ and $\nu = \sum_i c_i\nu_i$ be mixtures of the $\mu_i$'s and $\nu_i$'s with shared mixing weights $\{ c_i\}$. Then $W_\infty(\mu, \nu) \leq \max_{i\in\kset{n}} W_\infty(\mu_i, \nu_i)$.
\end{lemma}
\begin{proof}
For all $i\in\kset{n}$, let $\gamma_i$ be the coupling between $\mu_i$ and $\nu_i$ that achieves $W_\infty(\mu_i, \nu_i)$. 
Let the support of $\mu, \nu$ be $S$, and that of $\mu_i, \nu_i$ be $S_i$. We can extend each coupling $\gamma_i$ to have value $0$ at $(x,y) \in S\backslash S_i$.

We can construct a coupling $\gamma$ between $\mu$ and $\nu$ as follows: $\gamma(x, y) = \sum_{i=1}^n c_i \gamma_i(x, y)$.  Since this is a valid coupling, $W_\infty(\mu, \nu) \leq \max_{(x,y)\in A} |x - y|$
where $A = \{(x,y)|\gamma(x, y) \neq 0\}$. And we also know that $\max_{(x,y)\in A} |x - y| = \max_i W_\infty(\mu_i, \nu_i)$. Therefore we have $W_\infty(\mu, \nu) \leq \max_{i\in\kset{n}} W_\infty(\mu_i, \nu_i).$
\end{proof}

\section{Markov Quilt Mechanism}%

\subsection{General Properties}\label{sec:appendix_general_prop_mqm}

\begin{proof}(of Theorem~\ref{thm:mq})
Consider any secret pair $(s^i_a, s^i_b) \in \calQ$ and any $\theta \in \Theta$. Let $X_{Q,i}^*$ (with $X_{N,i}^*$, $X_{R,i}^*$) be Markov Quilt with minimum score for $X_i$. Since the trivial Markov Quilt $X_Q = \emptyset$ has $\eT(X_Q | X_i) = 0$ and $\sigma(X_Q) = n/\epsilon$, this score is $\leq n/\epsilon$. Below we use the notation $X_{Q}^*$, $X_{N}^*$, $X_{R}^*$ to denote $X_{Q,i}^*$, $X_{N,i}^*$, $X_{R,i}^*$ and use $X_{Q}^* \cup X_{R}^* = X^*_{R \cup Q}$.

For any $w$, we have $\frac{ P(F(X) + L \sigma_{\max} Z = w | X_i = a, \theta) }{P(F(X) + L \sigma_{\max} Z = w | X_i = b, \theta)}$ is at most:
\begin{align} \label{eqn:mqprivacy}
& \max_{x^*_{R \cup Q}} 
\frac{ P(F(X) + L \sigma_{\max} Z = w | X_i = a, X^*_{R \cup Q}=x^*_{R \cup Q}, \theta) } {P(F(X) + L \sigma_{\max} Z = w | X_i = b, X^*_{R \cup Q}=x^*_{R \cup Q}, \theta) } \nonumber\\
& \qquad\qquad\qquad\qquad\quad \frac{P(X^*_{R \cup Q}=x^*_{R \cup Q} | X_i = a, \theta)}{P(X^*_{R \cup Q}=x^*_{R \cup Q} | X_i = b, \theta)}.
\end{align}

Since $F$ is $L$-Lipschitz, for a fixed $X^*_{R \cup Q}$, $F(X)$ can vary by at most $L\cdot\card{X^*_N}$; thus for $x^*_{R \cup Q}$, the first ratio in \eqref{eqn:mqprivacy} is $\leq e^{\epsilon - \eT(X^*_Q | X_i)}.$

Since $X^*_R$ is independent of $X_i$ given $X^*_Q$, the second ratio in \eqref{eqn:mqprivacy} is at most $e^{\eT(X^*_Q | X_i)}$. The theorem follows. 
\end{proof}

\begin{proof}(of Theorem~\ref{thm:seqcompose})
Let $L_k$ be the Lipschitz coefficient of $F_k$.
Consider any secret pair $(X_i=a,X_i=b)\in\calQ$. 
In a Pufferfish instantiation $(\calS,\calQ,\Theta)$, given $\epsilon$ and $S_{Q,i}$, the active Markov Quilt $X_{Q,i}^*$ for $X_i$ used by the Markov Quilt Mechanism is fixed. 
Therefore all $M_k$ use the same active Markov Quilt, and we denote it by $X_Q$, with corresponding $X_R,X_N$. 
Let $Z_k\sim\Lap(\sigma^{(k)})$ denote the Laplace noise added by the Markov Quilt Mechanism $M_k$ to $F_k(D)$. 
For any $k$, since $X_Q$ is the active Markov Quilt, we have $\sigma^{(k)} \geq \frac{L_k \card{X_N}}{\epsilon - \eT(X_Q |X_i)}$.
Let $X_\RQ=X_R\cup X_Q$. 
Then for any $\{w_k\}_{k=1}^K$, we have
\begin{align}\label{eqn:comp serial main}
&\frac{P(\forall k, F_k(X)+Z_k=w_k | X_i=a)}{P(\forall k, F_k(X)+Z_k=w_k | X_i=b)} \nonumber\\
=&\frac	{\int P(\forall k, F_k(X)+Z_k=w_k, X_\RQ = x_\RQ| X_i=a) d{x_\RQ}}
		{\int P(\forall k, F_k(X)+Z_k=w_k, X_\RQ = x_\RQ| X_i=b) d{x_\RQ}} \nonumber\\
=&\frac	{\int P(\forall k, F_k(X)+Z_k=w_k| X_\RQ = x_\RQ, X_i=a) }
		{\int P(\forall k, F_k(X)+Z_k=w_k| X_\RQ = x_\RQ, X_i=b) } \nonumber \\
&\qquad\qquad\qquad\qquad\qquad		\frac{P(X_\RQ = x_\RQ|X_i=a) d{x_\RQ}}{P(X_\RQ = x_\RQ|X_i=b) d{x_\RQ}}	\nonumber\\		
\leq&\max_{x_\RQ}\frac	{P(\forall k, F_k(X)+Z_k=w_k| X_\RQ = x_\RQ, X_i=a) }
						{P(\forall k, F_k(X)+Z_k=w_k| X_\RQ = x_\RQ, X_i=b) } \nonumber\\
&\qquad\qquad\qquad\qquad\qquad	\frac{P(X_\RQ = x_\RQ|X_i=a)}{P(X_\RQ = x_\RQ|X_i=b)}
\end{align}
Let $X_{N\backslash\{i\}} = X_N\backslash X_i$. The first ratio in \eqref{eqn:comp serial main} equals to
\begin{align*}
=&\frac	{\int P(\forall k, F_k(X)+Z_k=w_k, {X_{N\backslash\{i\}} = x_{N\backslash\{i\}}}}
		{\int P(\forall k, F_k(X)+Z_k=w_k, {X_{N\backslash\{i\}} = x_{N\backslash\{i\}}}} \\
&\qquad\qquad\qquad\qquad\qquad
  \frac{| X_\RQ = x_\RQ, X_i=a) d {x_{N\backslash\{i\}}}}{| X_\RQ = x_\RQ, X_i=b) d {x_{N\backslash\{i\}}}} \\
=&\frac	{\int P(\forall k, F_k(X)+Z_k=w_k|{X_{N\backslash\{i\}} = x_{N\backslash\{i\}}}, X_\RQ = x_\RQ,}
		{\int P(\forall k, F_k(X)+Z_k=w_k|{X_{N\backslash\{i\}} = x_{N\backslash\{i\}}}, X_\RQ = x_\RQ,} \\
&\qquad\qquad
\frac	{X_i=a) P({x_{N\backslash\{i\}}}|X_\RQ = x_\RQ,X_i=a) d {x_{N\backslash\{i\}}}}
		{X_i=b) P({x_{N\backslash\{i\}}}|X_\RQ = x_\RQ,X_i=b) d {x_{N\backslash\{i\}}}}.
\end{align*}
Let $F_k(a,{x_{N\backslash\{i\}}},x_\RQ)$ denote the value of $F_k(X)$ with $X_i=a$, $X_\RQ = x_\RQ$ and ${X_{N\backslash\{i\}}} = {x_{N\backslash\{i\}}}$. 
Since $F_k(X)$'s are fixed given a fixed value of $X$, and $Z_k$'s are independent, the above equals to
\begin{align*}
=&\frac	{\int \Pi_k P(Z_k=w_k-F_k(a,{x_{N\backslash\{i\}}},x_\RQ)) }
		{\int \Pi_k P(Z_k=w_k-F_k(b,{x_{N\backslash\{i\}}},x_\RQ)) }\\
&\qquad\quad
\frac	{ P({X_{N\backslash\{i\}}}={x_{N\backslash\{i\}}}|X_\RQ = x_\RQ,X_i=a) d {x_{N\backslash\{i\}}}}
		{ P({X_{N\backslash\{i\}}}={x_{N\backslash\{i\}}}|X_\RQ = x_\RQ,X_i=b) d {x_{N\backslash\{i\}}}} \\
\leq&\frac	{\Pi_k \max_{x_{N\backslash\{i\}}} P(Z_k=w_k-F_k(a,{x_{N\backslash\{i\}}},x_\RQ)) }
			{\Pi_k \min_{x_{N\backslash\{i\}}} P(Z_k=w_k-F_k(b,{x_{N\backslash\{i\}}},x_\RQ)) }\\
&\qquad
\frac	{\int P({X_{N\backslash\{i\}}}={x_{N\backslash\{i\}}}|X_\RQ = x_\RQ,X_i=a) d {x_{N\backslash\{i\}}}}
		{\int P({X_{N\backslash\{i\}}}={x_{N\backslash\{i\}}}|X_\RQ = x_\RQ,X_i=b) d {x_{N\backslash\{i\}}}}
\end{align*}
Since $P({X_{N\backslash\{i\}}}|X_\RQ = x_\RQ,X_i=a)$, $P({X_{N\backslash\{i\}}}|X_\RQ = x_\RQ,X_i=b)$ are probability distributions which integrate to $1$, the above equals to
\begin{align*}
&\frac	{\Pi_k \max_{x_{N\backslash\{i\}}} P(Z_k=w_k-F_k(a,{x_{N\backslash\{i\}}},x_\RQ)) }
			{\Pi_k \min_{x_{N\backslash\{i\}}} P(Z_k=w_k-F_k(b,{x_{N\backslash\{i\}}},x_\RQ)) }		.	
\end{align*}
Notice that $F_k$ can change by at most $L_k\cdot \card{N}$ when $X_{N\backslash\{i\}}$ and $X_i$ change. So for any ${x_{N\backslash\{i\}}}, {x'_{N\backslash\{i\}}}$,
\begin{align*}
&\frac {P(Z_k=w_k-F_k(a,{x_{N\backslash\{i\}}},x_\RQ))} {P(Z_k=w_k-F_k(b,{x'_{N\backslash\{i\}}},x_\RQ))} \leq e^{\epsilon - \eT(X_Q|X_i)}.
\end{align*}
Therefore the first ratio in \eqref{eqn:comp serial main} is upper bounded by $\Pi_k e^{\epsilon - \eT(X_Q|X_i)}.$

As has been analyzed in the proof to Theorem~\ref{thm:mq}, the second ratio is bounded by $e^{\eT(X_Q|X_i)}$.
Combining the two ratios together, \eqref{eqn:comp serial main} is upper bounded by 
$\Pi_k e^{\epsilon - \eT(X_Q|X_i)} e^{\eT(X_Q|X_i)} \leq e^{K \epsilon},$
and the theorem follows.
\end{proof}
\subsection{Markov Chains}

\begin{proof}(of Lemma~\ref{lem:mqmcminimal})
Consider a Markov Quilt $(X_N, X_Q, X_R)$ of $X_i$. $X_Q = X_{Q,l} \cup X_{Q,r}$ where all nodes in $X_{Q,l}$ have index smaller than $i$, and all nodes in $X_{Q,r}$ have index larger than $i$. When $X_Q$ is non-empty, there are three cases.

First, both $X_{Q,l}$ and $X_{Q,r}$ are non-empty. Then there exist positive integers $a,b$, such that $X_{i-a} \in X_{Q,l}$ is the node with the largest index in $X_{Q,l}$, and $X_{i+b}$ is the node with the smallest index in $X_{Q,r}$. 
Since $X_{Q,l}\backslash X_{i-a}$ is independent of $X_i$ given $X_{i-a}$, and $X_{Q,r}\backslash X_{i-b}$ is independent of $X_i$ given $X_{i-b}$, $\eT(X_Q | X_i) = \eT(\{X_{i-a},X_{i+b}\} | X_i)$. 
Also, all nodes in $\{X_{i-a+1},\dots,X_{i+b-1}\}$ should be included in $X_N$, since they are not independent of $X_i$ given $X_Q$, and thus $\card{X_N} \geq a+b-1$. Thus, if $X_{Q'} = \{X_{i-a},X_{i+b}\} \in S_{Q,i}$ with $X_{N'} = \{X_{i-a+1},\dots,X_{i+b-1}\}$, then $\sigma(X_{Q'})\leq \sigma(X_Q)$.

The second case is when $X_{Q,r}$ is empty but $X_{Q,l}$ is not. Still, there exists positive integer $a$ such that $X_{i-a} \in X_{Q,l}$ is the node with largest index in $X_{Q,l}$. Since $X_{Q,l}\backslash X_{i-a}$ is independent of $X_i$ given $X_{i-a}$, $\eT(X_Q | X_i) = \eT(\{X_{i-a}\} | X_i)$. Since all nodes in $\{X_{i-a+1},\dots,X_{T}\}$ are not independent of $X_i$ given $X_Q$, they should be included in $X_N$, and thus $\card{X_N} \geq T-i+a$. 
Now consider $X_{Q'} = \{X_{i-a}\} \in S_{Q,i}$ with $X_{N'} = \{X_{i-a+1},\dots,X_{T}\}$. The above implies that $\sigma(X_{Q'})\leq \sigma(X_Q)$. The third case is when $X_{Q,l}$ is empty but $X_{Q,r}$ is not; this is analogous to the second case.
\end{proof}

Before proving Lemma~\ref{lem:mqmc2}, we overload the notation $\gT$ in \eqref{def:g} to capture the case where $\Theta$ consists of irreducible, aperiodic and \emph{reversible} Markov chains:
\begin{align}\label{def:g2}
&\gT = \\
&\begin{cases}
2\min_{\theta \in \Theta} \min \{ 1-|\lambda| : P_{\theta} x = \lambda x, |\lambda| < 1 \}, \theta\in\Theta\text{ reversible}\\
\min_{\theta \in \Theta} \min \{ 1-|\lambda| : P_{\theta} P^*_{\theta} x = \lambda x, |\lambda| < 1 \}\text{, otherwise.} \nonumber
\end{cases}
\end{align}
Now we restate Lemma~\ref{lem:mqmc2} to provide a better upper bound when $P_\theta, \forall\theta\in\Theta$ is irreducible, aperiodic, and \emph{reversible}.
\begin{lemma} \label{lem:mqmc2_full}
Let $\piT$ be as in \eqref{def:pimin}, and $\gT$ be as in \eqref{def:g2}. Suppose $a, b$ are integers such that $\min(a, b) \geq \frac{2\log(1/\piT)}{\gT}$. Then, for $X_Q = \{ X_{i - a}, X_{i + b} \}$, $\eT(X_Q | X_i)$ is at most:
\begin{align*} 
\log \left( \frac{ \piT + \exp(-\gT b/2) }{ \piT - \exp(-\gT b/2) } \right)
 +  2 \log \left( \frac{ \piT + \exp(-\gT a/2) }{ \piT - \exp(-\gT a/2) } \right);
\end{align*}
for $X_Q = \{ X_{i-a} \}$, $\eT(X_Q | X_i)  \leq 2 \log \left( \frac{ \piT + \exp(-\gT a/2) }{ \piT - \exp(-\gT a/2) } \right)$;\\
for $X_Q = \{ X_{i+b} \}$, $\eT(X_Q | X_i)  \leq \log \left( \frac{ \piT + \exp(-\gT b/2) }{ \piT - \exp(-\gT b/2) } \right),$ and for $X_Q = \emptyset$,  $\eT(X_Q | X_i)  = \ 0.$
\end{lemma}

The main ingredient in the proof of Lemma~\ref{lem:mqmc2} is from the following standard result in Markov Chain theory \cite{markovchaintheory}. 

\begin{lemma} 
\label{lem:expdecay}
Consider an \emph{aperiodic}, \emph{irreducible} $k$-state discrete time Markov Chain with transition matrix $P$.
Let $\pi$ be its stationary distribution and let $\pi^{\min} = \min_x \pi(x)$. Let $P^*$ be the time-reversal of $P$, and let $\g$ be defined as follows:
\begin{align*}
\g=
\begin{cases}
2\min \{ 1-|\lambda| : P x = \lambda x, |\lambda| < 1 \} \text{, }P\text{ is reversible}\\
\min \{ 1-|\lambda| : PP^* x = \lambda x, |\lambda| < 1 \} \text{, otherwise}.
\end{cases}
\end{align*}
If $\Delta_t = \frac{\exp(-t {\g}/2)}{\pi^{\min}}$, then $\left|\frac{P^t(x,y)}{\pi(y)} - 1\right| \leq \Delta_t$ for $t\geq \frac{2\log(1/\pi^{\min})}{\g}$.
\end{lemma}

This lemma, along with some algebra and an Bayes Rule, suffices to show the following.
\begin{lemma} \label{lem:fwdmc bwdmc}
For any $\Theta$ consists of irreducible and aperiodic Markov chains, 
let $\piT$ be defined as in \eqref{def:pimin}, and $\gT$ be defined as in \eqref{def:g2}. Let $\Delta_t = \frac{\exp(-t \gT / 2)}{\piT}$.
Suppose $t > \frac{2\log(1/\piT)}{\gT}$. 
For any $\theta \in \Theta$ and any $x$, $x'$ and $y$,
\[ \frac{1 - \Delta_t}{1 + \Delta_t} \leq \frac{\Pr(X_{t+j} = y | X_{j} = x, \theta)}{\Pr(X_{t+j} = y | X_{j} = x', \theta)} \leq \frac{1 + \Delta_t}{1 - \Delta_t}. \]
\[ \left(\frac{1 - \Delta_t}{1 + \Delta_t}\right)^2 \leq \frac{\Pr(X_{j} = y | X_{j+t} = x, \theta)}{\Pr(X_{j} = y | X_{j+t} = x', \theta)} \leq \left(\frac{1 + \Delta_t}{1 - \Delta_t}\right)^2. \]
\end{lemma}

\begin{proof}(of Lemma~\ref{lem:mqmc2})
Consider $X_Q$ of the form $\{X_{i-a},X_{i+b}\}$. For any $\theta$, any $x,x'$ and any $x_Q$, by conditional independence, $\log \frac{\Pr(X_Q = x_Q | X_i = x, \theta)}{\Pr(X_Q = x_Q | X_i = x', \theta)}$ is equal to:

\begin{align}\label{eqn:mqmc2_proof}
\log\frac{\Pr(X_{i+b}=x_{i+b}|X_i = x, \theta)}{\Pr(X_{i+b}=x_{i+b}|X_i = x', \theta)} + \log \frac{\Pr(X_{i-a}=x_{i-a}|X_i = x, \theta)}{\Pr(X_{i-a}=x_{i-a}|X_i = x', \theta)} \nonumber
\end{align}
By Lemma \ref{lem:fwdmc bwdmc}, this is at most $2\log\frac{1+\Delta_a}{1-\Delta_a} + \log\frac{1+\Delta_b}{1-\Delta_b}$ for any $x, x', x_{i-a}, x_{i+b}$.

Moreover, when $X_Q=\{X_{i-a}\}$, the first term disappears and $\log \frac{\Pr(X_Q = x_Q | X_i = x, \theta)}{\Pr(X_Q = x_Q | X_i = x', \theta)}$  is at most $2\log\frac{1+\Delta_a}{1-\Delta_a}$; the case $X_Q=\{X_{i+b}\}$ is analogous, thus concluding the proof.
\end{proof}

\begin{proof}(of Theorem~\ref{thm:mqm_utility_lemma})
Let $a^*$ be defined as in Lemma~\ref{lem:fastmqm}.
Consider the middle node $X_{\lceil T/2 \rceil}$ and Markov Quilt $X_Q=\{X_{\lceil T/2 \rceil-a^*}, X_{\lceil T/2 \rceil+a^*}\}$; since $T \geq 2a^*+1$, both endpoints of the quilt lie inside the chain. From Lemma \ref{lem:mqmc2} and the definition of $a^*$, $\eT(X_Q | X_i) \leq \epsilon/2$.

Let $C = 8 \lceil {\log\left(\frac{\exp({\epsilon/6})+1}{\exp({\epsilon/6})-1} \frac{1}{\piT}\right)}/{\gT} \rceil$. We have $\sigma(X_Q) \leq \frac{C}{\epsilon}.$ By definition, $\sigma_{\lceil T/2 \rceil} \leq \sigma(X_Q)$. Also, by Lemma~\ref{lem:mqm fast}, $\sigma_j \leq \sigma_{\lceil T/2 \rceil}$ for any $j$. Therefore $\max_{j\in\kset{T}}\sigma_j \leq C/\epsilon$.
\end{proof}

\subsection{Fast \mmqmapprox}
\label{sec:app: mqm fast}
\begin{lemma}\label{lem:mqm fast}
In Algorithm \mmqmapprox, suppose there exists an $i \in \kset{T}$, such that $\sigma_i = \min_{X_Q \in S_{Q,i}} \sigma(X_Q)$ is achieved by $X_Q = \{ X_{i-a}, X_{i+b}\}$ with $i-a \geq 1$, $i+b\leq T$. Then, $\max_{j\in\kset{T}} \sigma_j = \sigma_i$.
\end{lemma}

\begin{proof}(Of Lemma~\ref{lem:mqm fast})
Let $X_{Q} = \{x_{i-a}, x_{i+b}\}$ be the quilt with the lowest score for node $X_i$. Pick any other node $X_{i'}$. If both $i' - a \geq 1$ and $i' + b \leq T$, then $\{X_{i'-a}, X_{i'+b}\}$ is a quilt for $X_{i'}$ with score $\sigma_i$. Since $\sigma_{i'}$ is the minimum score over all quilts of $X_{i'}$, $\sigma_{i'} \leq \sigma_i$.

Otherwise, at least one of the conditions $i' - a \geq 1$ or $i' + b \leq T$ hold. Suppose $i' - a\geq 1$. Then, consider the quilt $X_{Q'} = \{X_{i'-a}\}$ for $X_{i'}$. Let $X_N$ (resp. $X_{N'}$) be the set of local nodes of $X_i$ (resp. $X_{i'}$) corresponding to $X_Q$ (resp. $X_{Q'}$). Since $i'+b \geq T+1$, we have $\card {X_{N'}} = T-i'+a \leq a+b-1 = \card{X_N}$. This implies that $e_\Theta(X_{Q'}|X_{i'})$ is equal to:
\begin{align*}
& = 2 \log \left( \frac{ \piT + \exp(-\gT a/2) }{ \piT - \exp(-\gT a/2) } \right)\\
& \leq \log \left( \frac{ \piT + \exp(-\gT b/2) }{ \piT - \exp(-\gT b/2) } \right) +  2 \log \left( \frac{ \piT + \exp(-\gT a/2) }{ \piT - \exp(-\gT a/2) } \right),
\end{align*}
which is $e_\Theta(X_{Q}|X_{i})$. This mean that $\sigma(X_{Q'}) = \frac{\card{X_{N'}}}{\epsilon - e_{\Theta}(X_{Q'}|X_{i'})}$ is at most  $\frac{\card{X_N}}{\epsilon - e_{\Theta}(X_Q|X_{i})} = \sigma_i$. Therefore, $\sigma_{i'} \leq \sigma_i$.

The other case where $i' + b \leq T$ is analogous. Thus $\forall i'$, $\sigma_{i'}\leq \sigma_i$ and $\sigma_{\max} = \sigma_i$.
\end{proof}

\begin{proof} (Of Lemma~\ref{lem:fastmqm})
Consider the middle node $X_{\lceil T/2 \rceil}$ and Markov Quilt $X_Q=\{X_{\lceil T/2 \rceil-a^*}, X_{\lceil T/2 \rceil+a^*}\}$; since $T \geq 8a^*$, both endpoints of the quilt lie inside the chain. From Lemma \ref{lem:mqmc2} and the definition of $a^*$, $\eT(X_Q | X_i) \leq \epsilon/2$. Therefore, $\sigma(X_{Q}) = \frac{2a^*-1}{\epsilon - \eT(X_{Q} | X_i)} \leq \frac{4a^*-2}{\epsilon}$.

Consider any Markov Quilt $X_{Q'}$ (with corresponding $X_{N'}$ and $X_{R'}$) with $\card{X_{N'}} \geq 4a^*-2$. Since \effect\ is always non-negative, we have $\sigma(X_{Q'}) = \frac{\card{X_{N'}}}{\epsilon - \eT(X_{Q'} | X_i)} \geq \frac{4a^*-2}{\epsilon} \geq \sigma(X_{Q})$, and thus any Markov Quilt $X_{Q'}$ with $\card{X_{N'}} \geq 4a^*-2$ has score no less than that of $X_Q=\{X_{i-a^*}, X_{i+a^*}\}$. If $X_{Q'}$ is of the form $\{X_{i-a'}\}$ or $\{X_{i+b'}\}$, then since $T \geq 8a^*$, we have $\card{X_{N'}} \geq 4a^*$. Therefore the optimal Markov Quilt for $X_i$ is of the form $\{X_{i-a}, X_{i+b}\}$ with $a+b \leq 4a^*$. Combining this with Lemma~\ref{lem:mqm fast} concludes the proof.
\end{proof}

\subsection{\mmqmexact\ Optimization}\label{sec:mqmexactallinit}
We show that \mmqmexact\ can be expedited further if $\Theta$ is of the form: $\Theta = \Delta_k \times \calP$, where $\calP$ is a set of transition matrices, and $\Delta_k$ is the probability simplex over $k$ items. In other words, $\Theta$ includes tuples of the form $(q, P)$ where $q$ ranges over possible initial distributions and $P$ belongs to a set $\calP$ of transition matrices. 

For any $P \in \calP$, let $\Theta_P = \{ (P, q) \in \Theta\}$. Let $P^j$ be $P$ raised to the $j$-th power. From~\eqref{eqn:maxinfluence_exact}, $e_{\Theta_P}(X_Q | X_i)$ is equal to:
\vspace{-10pt}
\begin{align*}
&	\max_{x, x'\in \calX} \left(\log \max_{q}\frac{(q^\top P^{i-1})(x')}{(q^\top P^{i-1})(x)}  \right.\\
&+\max_{x_{i+b}\in\calX} \log \frac{P^{b}(x, x_{i+b})}{P^{b}(x', x_{i+b})}
+ \left.\max_{x_{i-a}\in\calX} \log \frac{P^{a}(x_{i-a},x)}{P^{a}(x_{i-a},x')}\right).
\end{align*}
Since $(q^\top P^{i-1})(x) = q^\top P^{i-1}(:,x)$, for $i>1$, we have
\begin{align*}
	\max_{q}\frac{(q^\top P^{i-1})(x')}{(q^\top P^{i-1})(x)} 
= 	\max_{q}\frac{q^\top P^{i-1}(:,x')}{q^\top P^{i-1}(:,x)} 
=	\max_{y}\frac{P^{i-1}(y,x')}{P^{i-1}(y,x)} 
\end{align*}
where the the maximum is achieved by $q = e_{y^*}$, %
$y^* = \arg\max_{y}\frac{P^{i-1}(y,x')}{P^{i-1}(y,x)}$. \\

This implies that we only need to iterate over all transition matrices $\calP$ instead of conducting a grid search over $\Theta$, which further improves efficiency.
\end{document}